\documentclass{article}

\usepackage{iclr2027_conference}
\iclrfinalcopy
\usepackage{times}

\usepackage{graphicx} %
\usepackage{notation}
\usepackage{tikz}
\usetikzlibrary{tikzmark, calc, backgrounds}
\usepackage{enumitem}
\usepackage{wrapfig}
\usepackage{amsthm}
\usepackage{amssymb}
\usepackage{natbib}
\usepackage{caption}
\usepackage{multirow}
\usepackage{booktabs}
\usepackage{array}
\usepackage{cancel}
\usepackage[ruled, vlined, linesnumbered]{algorithm2e}
\usepackage[american]{babel}
\usepackage{mathtools}
\usepackage[dvipsnames,table]{xcolor}
\usepackage[most]{tcolorbox}
\usepackage[toc, page, header]{appendix}
\usepackage{minitoc}
\usepackage{listings}
\usepackage{float} %
\floatstyle{ruled}
\newfloat{listing}{tbp}{lol}
\floatname{listing}{Listing}
\usepackage{longtable}
\usepackage{adjustbox}
\usepackage{etoolbox}
\usepackage{csquotes}
\usepackage{makecell}
\usepackage[backref=page,colorlinks=true,citecolor=blue]{hyperref}
\usepackage{cleveref}

\renewcommand{\mkbegdispquote}[2]{\itshape}

\newcommand{\revised}[1]{#1}

\tcbuselibrary{listingsutf8}

\crefname{figure}{Fig.}{Figs.}
\crefname{definition}{Def.}{Defs.}
\crefname{corollary}{Cor.}{Cors.}
\crefname{proposition}{Prop.}{Props.}
\crefname{observation}{Obs.}{Obs.}
\crefname{theorem}{Thm.}{Thms.}
\crefname{remark}{Remark}{Remarks}
\crefname{principle}{Principle}{Principles}
\crefname{lemma}{Lemma}{Lemmata}
\crefname{claim}{Claim}{Claims}
\crefname{table}{Tab.}{Tabs.}
\crefname{section}{Sec.}{Sec.}
\crefname{subsection}{Sec.}{Sec.}
\crefname{subsubsection}{Sec.}{Sec.}
\crefname{assumption}{Assumption}{Assumptions}
\crefname{appendix}{App.}{App.}
\crefname{algorithm}{Alg.}{Algs.}
\crefname{equation}{}{}
\crefname{enumi}{Step}{Steps}
\crefname{example}{Example}{Examples}

\newtheoremstyle{plain-sc}
  {6pt}   %
  {6pt}   %
  {\itshape}
  {}
  {\bfseries}
  {.}
  { }
  {}
\theoremstyle{plain-sc}
\newtheorem{proposition}{Proposition}

\def\app#1#2{%
  \mathrel{%
    \setbox0=\hbox{$#1\sim$}%
    \setbox2=\hbox{%
      \rlap{\hbox{$#1\propto$}}%
      \lower1.2\ht0\box0%
    }%
    \raise0.25\ht2\box2%
  }%
}

\title{Adaptive Inverted-Index Routing \\ for Granular Mixtures-of-Experts}

\author{%
  \parbox[t]{0.46\textwidth}{\vspace{0pt}\centering\normalfont
    \textbf{Klaus-Rudolf Kladny\textsuperscript{1,2}}\\[-0.25ex]
    {\footnotesize MPI for Intelligent Systems\\
    Tübingen AI Center\\
    \texttt{kkladny@tue.mpg.de}}}
  \And
  \parbox[t]{0.46\textwidth}{\vspace{0pt}\centering\normalfont
    \textbf{Maximilian Mordig\textsuperscript{1,2,3}}\\[-0.25ex]
    {\footnotesize MPI for Intelligent Systems\\
    Tübingen AI Center\\
    ETH Zurich\\
    \texttt{mmordig@tue.mpg.de}}}
  \AND
  \parbox[t]{0.46\textwidth}{\vspace{0pt}\centering\normalfont
    \textbf{Bernhard Schölkopf\textsuperscript{1,2,3,4}}\\[-0.25ex]
    {\footnotesize MPI for Intelligent Systems\\
    Tübingen AI Center\\
    ETH Zurich\\
    ELLIS Institute Tübingen\\
    \texttt{bs@tue.mpg.de}}}
  \And
  \parbox[t]{0.46\textwidth}{\vspace{0pt}\centering\normalfont
    \textbf{Michael Muehlebach\textsuperscript{1}}\\[-0.25ex]
    {\footnotesize MPI for Intelligent Systems\\
    \texttt{michaelm@tue.mpg.de}}}
}

\newcommand{\ourmethod}{AIR-MoE}

\definecolor{myblue}{RGB}{0, 106, 43} %
\colorlet{airrow}{myblue!8}
\definecolor{airtext}{RGB}{0,0,0}

\newcommand{\nicecomment}[1]{\hfill\textcolor{myblue}{\texttt{// #1}}}
\newcommand{\nicefullcomment}[1]{\textcolor{myblue}{\texttt{// #1}}}

\newtcolorbox{codebox}{
  enhanced,
  colback=myblue!5,
  boxrule=0.4pt, arc=2mm,
  left=1mm, right=1mm, top=0.6mm, bottom=0.6mm,
  before skip=6pt, after skip=6pt
}

\definecolor{codebg}{RGB}{248,248,250}
\definecolor{codeframe}{RGB}{210,210,220}
\definecolor{codekw}{RGB}{38,139,210}   %
\definecolor{codecm}{RGB}{88,110,117}   %
\definecolor{codestr}{RGB}{42,161,152}  %
\definecolor{codenum}{RGB}{101,123,131} %

\lstdefinestyle{pythonstyle}{
  language=Python,
  basicstyle=\ttfamily\scriptsize,
  keywordstyle=\bfseries\color{codekw},
  commentstyle=\itshape\color{codecm},
  stringstyle=\color{codestr},
  numberstyle=\tiny\color{codenum},
  numbers=none,
  stepnumber=1,
  numbersep=6pt,
  showstringspaces=false,
  breaklines=true,
  breakatwhitespace=true,
  columns=fullflexible,
  keepspaces=true,
  tabsize=2
}

\newtcblisting{breakablepython}{
  enhanced jigsaw,
  breakable,
  listing only,
  listing engine=listings,
  listing options={style=pythonstyle},
  colback=myblue!5,
  frame hidden,
  boxrule=0pt,
  arc=2mm,
  outer arc=2mm,
  left=1.5mm,
  right=1.5mm
}

\newtcolorbox{centralequation}[1][]{
  colback=myblue!6,
  rounded corners,
  boxrule=0.1pt,
  #1 %
}

\newtcolorbox{remarkbox}[1][]{
  colback=white,
  rounded corners,
  boxrule=0.5pt,
  #1 %
}

\newcolumntype{C}[1]{>{\centering\arraybackslash}p{#1}}

\newcommand{\algmark}[1]{\tikzmark{#1}}

\begin{document}

\renewcommand \partname{}
\renewcommand \thepart{}

\doparttoc
\faketableofcontents

\maketitle

\begin{abstract}
Sparse Mixture-of-experts (MoE) models enable scalable transformer architectures by activating only a subset of experts per token. Recent evidence suggests that performance improves with increasingly granular experts, i.e., many small experts instead of a few large ones. However, this regime substantially increases routing cost, which can dominate computation and memory. We introduce \textit{\underline{a}daptive \underline{i}nverted-index \underline{r}outing for MoE}~(\ourmethod{}), an inverted-index-inspired routing architecture based on vector quantization (VQ). In a first stage, \ourmethod{} performs \emph{coarse shortlisting} by assigning tokens to VQ codewords to construct a candidate set of experts. In a second stage, \emph{fine scoring} computes exact routing scores restricted to this shortlist. This two-stage procedure \textit{approximates} true top-$K$ routing while avoiding full expert scoring and, in contrast to prior work, imposing no structural constraints on expert parameters. \ourmethod{} serves as a drop-in replacement for standard routers and requires no modifications to the model architecture or loss function. We further provide a lower bound on the mass recall achieved by \ourmethod{} that yields insights into the inner workings. Empirically, \ourmethod{} improves the perplexity–FLOPs trade-off over existing MoE routers \revised{in regimes with up to more than a million tiny experts}, and achieves consistent perplexity improvements over the best existing baseline.
\end{abstract}

\section{Introduction}

Mixture-of-Experts (MoE;~\citet{jacobs1991adaptive}) is an ensemble learning technique where multiple so-called expert models are jointly trained with a routing mechanism that, for a given input, predicts weights to create a linear combination over expert outputs. Sparse MoE~\citep{shazeer2017outrageously} is a variant of MoE that has become popular in recent LLM architectures~(e.g.,~\citet{du2022glam, team2025kimi, jiang2024mixtral}). In a sparse MoE, only the top-$K$ experts (in terms of router scores) are considered in the linear combination, instead of all of them. This way, computational demands during inference are kept nearly constant as the number of experts grows, while increasing the expressivity of the model.

\begin{figure}[t]
    \centering
    \includegraphics[width=1.0\textwidth]{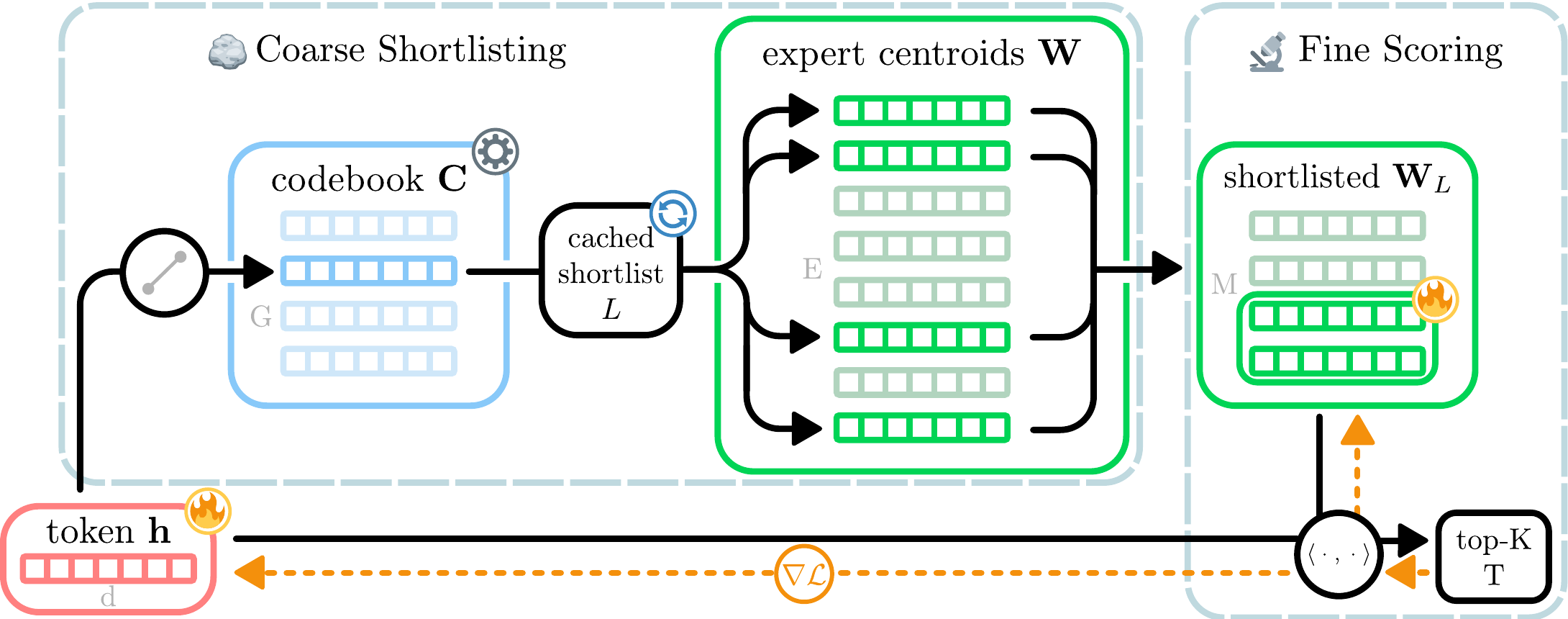}
    \caption{\textbf{Overview of \ourmethod{}}. The method consists of two parts: The coarse shortlisting stage uses vector quantization to select a codeword (blue) that stores a pre-computed expert shortlist $L$ that references specific expert centroids (green) and is updated after each optimizer step. The fine scoring stage takes the shortlisted expert weights and scores them exactly. Notably, the codebook is learned via gradient-free optimization (gear symbol) and only token representations and expert centroids are trained using the downstream gradient (dashed orange) \textit{without} straight-through estimation trick.}
    \label{fig:overview}
\end{figure}

Empirical scaling laws suggest that the ideal regime uses \textit{many} experts with \textit{few} parameters each—so-called \textit{granular experts}~\citep{krajewski2024scaling}.\footnote{We note that similar observations have been made earlier by~\citet{clark2022unified}.} In practice, training such models is prohibitive because the routing overhead \revised{(computational and memory)} becomes non-negligible as the number of experts increases. In recommender systems, this problem is known as maximum inner product search (MIPS;~\citet{shrivastava2014asymmetric,abuzaid2019index}): find the maximum (or top-$K$) inner product value(s) between a query vector (token representation) and a large number of candidate vectors (expert centroids). Prior work in MoE research addresses this challenge either by restricting routing to predefined groups (hierarchical MoE;~\citet{shazeer2017outrageously}) or by imposing structural constraints on expert representations~\citep{he2024mixture}. While these approaches reduce computational \revised{and memory costs}, they limit routing flexibility~(see \cref{appx:expert_usage}), which typically degrades performance.

In the present work, we propose to use vector quantization to construct an inverted-index-like routing architecture, visualized in~\cref{fig:overview}. For every token, in parallel, the top-$K$ prediction procedure consists of two steps:
\begin{enumerate}[leftmargin=*]
    \item \underline{Coarse Shortlisting}: Retrieve the closest codeword and gather expert indices from the corresponding pre-evaluated shortlist.
    \item \underline{Fine Scoring}: Score all inner products for token and experts within the retrieved shortlist, then use the top-$K$ expert indices and scores for routing. 
\end{enumerate}

This procedure avoids structural constraints: experts are neither partitioned (shortlists can be overlapping) nor parametrically restricted. Instead of enforcing exact top-$K$ routing via constraints, we deliberately pursue an efficient and flexible approximation, akin to approximate nearest neighbor search~\citep{indyk1998approximate}.

A central challenge that arises from the two-step procedure sketched above is that the coarse shortlisting step is not differentiable. To address this issue, we propose a bi-level optimization approach where the codebook training is gradient-free and decoupled from the training of all other model parameters. Specifically, the codebook is trained via an adaptive k-means that gradually adapts to the internal distribution shift.

In addition to the methodological contribution, we explore the theoretical assumptions under which the top-$K$ approximation of \ourmethod{} is reasonable in terms of retaining probability mass in the routed shortlist (mass recall). Specifically, we show a lower bound on the mass recall that depends on the quantization error and the routing mass that lies outside of the top-M codeword scores.

We summarize our main contributions as follows:

\begin{itemize}[leftmargin=*]
        \item We introduce the \textit{adaptive inverted-index for MoE} (\ourmethod{}), an inverted-index-like routing architecture for granular mixture-of-experts based on adaptive vector quantization that approximates the underlying top-$K$ experts without imposing restrictions~(\cref{sec:cf-moe}).
        \item To address the non-differentiability of the codebook, we propose a bi-level optimization strategy that updates the router parameters via gradient descent while learning the codebook using gradient-free adaptive $k$-means~(\cref{sec:training}).
        \item We derive a lower bound on the retained total probability mass of \ourmethod{}, thereby providing an intuition for why and under what conditions the methods can be expected to work well~(\cref{sec:theory}).
        \item We demonstrate empirically that \ourmethod{} tends to perform favorably in comparison to coarse-grained MoE and existing routing methods for granular MoE~(\cref{sec:experiments}).
\end{itemize}

\textbf{Overview.}~\cref{sec:prelim} covers the main preliminaries of the manuscript, namely mixture-of-experts as used in modern LLMs~(\cref{sec:moe_prelim}) and granular MoE~(\cref{sec:granular_prelim}), the motivating scaling law for \ourmethod{}. Thereafter, in~\cref{sec:cf-moe}, we propose \ourmethod{}, specifically forward pass~(\cref{sec:forward_pass}), training~(\cref{sec:training}) and theory~(\cref{sec:theory}). We then cover related work~(\cref{sec:related_work}) and experiments~(\cref{sec:experiments}). Finally, we discuss limitations in~\cref{sec:discussion} and conclude with~\cref{sec:conclusion}.

\section{Preliminaries} \label{sec:prelim}

\subsection{Mixture-of-Experts in LLMs} \label{sec:moe_prelim}

In transformer-based LLMs~\citep{vaswani2017attention}, MoEs are typically applied to feed-forward neural network (FFN) layers with top-$K$ inference~\citep{shazeer2017outrageously} for compute \& memory efficiency:
\begin{equation}
\label{eq:MoE}
    \MoE(\hb)
    \;\coloneqq\;
    \sum_{e \in \mathrm{TopK}\{\gammab(\hb)\}}
    \gamma_e(\hb)\,\FFN_{\phi}^{(e)}(\hb),
\end{equation}
where $\mathrm{TopK}$ denotes the indices of the $k \in \NN$ largest router weights in $\gammab(\hb) \in \RR^E$. The router weight for expert $e$ is defined as
\begin{equation}
\label{eq:expert_score}
    \gamma_e(\hb)
    \;\coloneqq\;
    \frac{\exp\{z_e(\hb)\}}
    {\sum_{e' \in \mathrm{TopK} \{\zb(\hb)\}} \exp\{z_{e'}(\hb)\}}, \quad z_e(\hb)
    \;\coloneqq\;
    \langle \wb_e, \hb \rangle,
\end{equation}
where $\wb_e \in \RR^d$ for all $e \in \{1,2,\dots,E\}$ are learnable expert centroids. A common choice for the FFN parameterization in \cref{eq:MoE} is the gated linear unit architecture of~\citet{shazeer2020glu}.

\subsection{Granular MoE} \label{sec:granular_prelim}

Empirical scaling laws suggest that increasing expert granularity (using many small experts instead of a few large ones) can improve performance~\citep{krajewski2024scaling}. In particular,
\begin{equation}
\label{eq:scaling_law}
    \Lcal(Q)
    \;=\;
    \frac{a}{Q^b} + c,
    \qquad
    Q \;\coloneqq\; \frac{d_{\mathrm{standard}}}{d_{\mathrm{expert}}},
\end{equation}
where $Q$ denotes the granularity with reference dimension $d_{\mathrm{standard}}$ and expert dimension $d_{\mathrm{expert}}$. The variables $a, b, c$ are positive constants. However, this regime substantially increases routing cost, as the router must evaluate inner products $\langle \wb_e, \hb \rangle$ across many experts.

The next section introduces our approach for reducing this routing overhead.

\section{Adaptive Inverted-Index Routing for Mixture-of-Experts (\ourmethod{})}
\label{sec:cf-moe}
\newcommand{\stagelabel}[1]{\hfill\colorbox{gray!15}{\strut\textbf{#1}}}

\begin{figure}[t]
\centering
\begin{minipage}{0.945\textwidth}
\begin{algorithm}[H]
\LinesNumbered
\IncMargin{0.5em}
\caption{\ourmethod{} Forward Pass} \label{alg:cf_router}
\KwIn{Token batch $\Hb = \{\hb_s\}_{s=1}^S$, codebook $\Cb = \{\cb_g\}_{g=1}^G$, expert centroids $\Wb$, shortlist size $M$, $K$ (for top-$K$ selection), jitter $\sigma_{\mathrm{jit}} > 0$}
\KwOut{Selected expert indices $T_s$ and scores $\zb_s$}

\BlankLine
$\wb_e \gets \Pi_{\mathbb{S}^{d-1}}(\wb_e)$, \quad $\forall e \in [E]$ \nllabel{ln:normalize_centroid} \nicecomment{normalize expert centroids} 
\BlankLine

\algmark{coarse-start}
\nllabel{ln:coarse_1}
\For{$s \leftarrow 1$ \KwTo $S$}{
  $g_s \leftarrow \arg\min_{g \in [G]} D (\hb_s, \cb_g)$ \nllabel{ln:codeword_mapping}
}

\If(\nicecomment{refresh only after optimizer update}){shortlist cache is invalid}{
  \For{$g \leftarrow 1$ \KwTo $G$}{ \nllabel{ln:shortlist_refresh_1}
    $\etab_1 \sim \Ncal(\zerob, \sigma_{\mathrm{jit}}^2 \Ib)$ \nllabel{ln:noise_1} \\
    $L_g \leftarrow \mathrm{TopM}_{e \in [E]}\big(\langle \cb_g, \wb_e \rangle + \etab_1 \big)$ \nllabel{ln:codeword_expert}
  }
}
\nllabel{ln:coarse_2}
\algmark{coarse-end}
\BlankLine
\algmark{fine-start}
\nllabel{ln:fine_1}
\For{$s \leftarrow 1$ \KwTo $S$}{
  $\etab_2 \sim \Ncal(\zerob, \sigma_{\mathrm{jit}}^2 \Ib)$ \nllabel{ln:noise_2} \\
  $T_s, \zb_s \leftarrow \mathrm{TopK}_{e \in L_{g_s}}\big(\langle \hb_s, \wb_e \rangle + \etab_2 \big)$ \nllabel{ln:scoring} \nicecomment{gradients do not pass to $\Cb$}
}
\nllabel{ln:fine_2}
\algmark{fine-end}
\BlankLine
\begin{tikzpicture}[overlay, remember picture]

\begin{pgfonlayer}{background}
  \fill[myblue!50, rounded corners=3pt, fill opacity=0.3] %
    ($(pic cs:coarse-start)+(-0.3em,2.0ex)$)
    rectangle
    ($(pic cs:coarse-end)+(\linewidth,0ex)$);

  \fill[myblue!50, rounded corners=3pt, fill opacity=0.3] %
    ($(pic cs:fine-start)+(-0.3em,2.0ex)$)
    rectangle
    ($(pic cs:fine-end)+(\linewidth,0ex)$);
\end{pgfonlayer}

\node[anchor=north east, font=\bfseries, text=black]
  at ($(pic cs:coarse-start)+(\linewidth,1.5ex)$) {Coarse Shortlisting};

\node[anchor=north east, font=\bfseries, text=black]
  at ($(pic cs:fine-start)+(\linewidth,1.5ex)$) {Fine Scoring};

\end{tikzpicture}
\Return{$\{T_s, \zb_s\}_{s=1}^S$}
\end{algorithm}
\end{minipage}
\end{figure}

We proceed to elaborate on both the \ourmethod{} forward pass in~\cref{sec:forward_pass} and how the router is trained in~\cref{sec:training}.

\subsection{Forward Pass} \label{sec:forward_pass}

The full forward pass of an \ourmethod{} router is shown in~\cref{alg:cf_router}. We stress that in practice, we parallelize~\cref{alg:cf_router} over tokens and codewords and \textit{do not require looping} (see~\cref{appx:implementation_details}). 

\paragraph{Two-Stage Retrieval.} The essence is a retrieval-style routing that first shortlists a set of candidate experts $L_{g_s}$ (``coarse'') before evaluating the final expert scores within the shortlist set $T$ (``fine''), thereby reducing computational cost. Specifically, this leads to the subset hierarchy
\begin{equation}
    \underbrace{T}_{\text{top-$K$ index set}}
    \subset
    \underbrace{L}_{\text{shortlist}}
    \subset
    \underbrace{[E]}_{\text{full set of experts}}.
\end{equation}

\paragraph{Coarse Shortlisting~(\cref{alg:cf_router}, lines~\ref{ln:coarse_1}-\ref{ln:coarse_2}).}
To obtain a shortlist for an input token $\hb_s \in \mathbb{R}^d$ within a batch of tokens $\Hb \coloneqq \{ \hb_s \}_{s=1}^S$, we employ vector quantization~\citep{gray1984vector, linde1980algorithm}. Specifically, we map each input token $\hb_s$ to a codeword from a codebook $C = \{\cb_1,\dots,\cb_G\}$ of size $G$.\footnote{in practice, we note that $G\ll E$ to obtain computational gains.} We use Euclidean distance for the mapping~(\cref{alg:cf_router}, line~\ref{ln:codeword_mapping}):
\begin{equation} \label{eq:codeword_assignment}
    g_s \;\coloneqq\; \arg \min_{g \in [G]} \; D (\hb_s, \cb_g),
    \qquad
    D (\xb, \yb) \coloneqq \norm{\xb - \yb}_2.
\end{equation}
The assigned codeword is then given by $c(\hb_s) \coloneqq \cb_{g_s}$. Each codeword $\cb_g$ in turn stores a shortlist $L(\cb_g) $, which is refreshed after each optimizer step by scanning the entire set of experts~(\cref{alg:cf_router}, lines~\ref{ln:shortlist_refresh_1}-\ref{ln:codeword_expert}). This yields a learned inverted-index over experts, where codewords define coarse cells and each shortlist corresponds to a posting list of candidate experts. While refreshing shortlists incurs a computational cost of $\Ocal(EGd)$, the computation is amortized across the effective batch because it is independent of the number of tokens being routed $S$, making the arithmetic cost increasingly favorable in large effective-batch regimes common in pre-training.

\paragraph{Fine Scoring~(\cref{alg:cf_router}, lines~\ref{ln:fine_1}-\ref{ln:fine_2}).}
After a shortlist is retrieved, the final stage scores each token exactly within the shortlist. In this way, both an approximation $T_s \subset L_g$ of the top-$K$ indices and corresponding scores $\zb_s \in \mathbb{R}^k$ are obtained that represent the active experts and their selected router logits, respectively.

\paragraph{Implementation Details.} To prevent experts from starvation, we employ switch-transformer-style load balancing and jitter~(only at training time;~\citet{fedus2022switch}). In order to encourage exploration both on shortlist and token level, we apply independent jitter in both stages~(\cref{alg:cf_router}, lines~\ref{ln:noise_1},~\ref{ln:noise_2}). Another relevant aspect is the projection of expert centroids onto the unit sphere $\Pi_{\mathbb{S}^{d}}(\wb_e) \coloneqq \wb_e / \norm{\wb_e}$ (\cref{alg:cf_router}, line~\ref{ln:normalize_centroid}). Doing so has been shown to be beneficial by~\citet{nguyen2024statistical}. In addition, the normalization is important to ensure tightness of the centroid approximation~(\cref{appx:proof}, which proves \cref{prop:mass_preservation}). 

\subsection{Training} \label{sec:training}

Training \ourmethod{} separates the gradient-based optimization of model parameters from the gradient-free adaptation of the codebook. The token representations and expert centroids are updated through the downstream language-modeling objective, whereas the codebook tracks the evolving representation distribution via an adaptive $k$-means algorithm. We describe both components below and provide the full optimization loop in~\cref{appx:optimization_loop}.

\paragraph{Gradient-Free Codebook Training.} A key aspect in learning the codebook $\Cb$ is that no gradients flow from the router output to $\Cb$, because $\Cb$ is only used for shortlisting experts~(Alg.~\ref{alg:cf_router}, line~\ref{ln:codeword_expert}), while the actual score computation is done using the token representations $\Hb$~(Alg.~\ref{alg:cf_router}, line~\ref{ln:scoring}). To make sure that a good codebook is learned, we use an adaptive k-means algorithm~(\cref{alg:spherical_kmeans}; adaptation of~\citet{mcqueen1967some}) to minimize the expected quantization error. The proposed algorithm~(\cref{alg:spherical_kmeans}) combines two ideas:
\begin{enumerate}[leftmargin=*]
    \item \underline{Exponential moving average}~(\citet{van2017neural}; EMA;~\cref{alg:spherical_kmeans}: lines~\ref{lnn:ema_1},~\ref{lnn:ema_2}): To model the gradual drift of the distribution of internal representations, we use an exponential moving average approach. Notably, we perform EMA for both the codewords $\Mb$ and the counts $n$.
        \item \underline{Reinitialization of inactive codes}~(\citet{williams2020hierarchical};~\cref{alg:spherical_kmeans}: line~\ref{lnn:replacement}): A challenge in learning codebooks is that codewords can die, because non-selected codewords are not updated.\footnote{This phenomenon is analogous to dying experts in sparse MoE architectures.} We therefore replace codewords whose estimated average count falls below a threshold $\tau > 0$ by random token representations from the current batch.
\end{enumerate}

\begin{figure}[t]
\centering
\begin{minipage}{0.945\textwidth}
\begin{algorithm}[H] 
\LinesNumbered
\caption{Gradient-Free Codebook Update (Adaptive K-Means)} \label{alg:spherical_kmeans}
\KwIn{Token batch $\Hb \coloneqq \{\hb_s\}_{s=1}^S \subset \mathbb{R}^d$, codebook $\Cb = \{\cb_g\}_{g=1}^G \subset \mathbb{R}^d$, decay $\gamma \in [0,1)$, running counts $n \coloneqq \{n_g\}_{g=1}^G$, running sums $\Mb \coloneqq \{\mb_g\}_{g=1}^G \subset \mathbb{R}^d$, dead-code threshold $\tau$}
\KwOut{Updated codebook $\Cb$, running counts $n$, running sums $m$}

\BlankLine

\For{$s \leftarrow 1$ \KwTo $S$}{
  $g_s \leftarrow \arg\min_{g \in [G]} D(\hb_s, \cb_g)$
  \nicecomment{assign token to nearest codeword}
}

\BlankLine

\For{$g \leftarrow 1$ \KwTo $G$}{
  $n_g^{\mathrm{batch}} \leftarrow \sum_{s=1}^{S} \mathbf{1}[g_s = g]$
  \nicecomment{batch assignment count}
  
  $\mb_g^{\mathrm{batch}} \leftarrow \sum_{s=1}^{S} \mathbf{1}[g_s = g] \, \hb_s$
  \nicecomment{batch sum of assigned token states}

\BlankLine
$n_g \leftarrow \gamma n_g + (1-\gamma)n_g^{\mathrm{batch}}$
\nllabel{lnn:ema_1}
\nicecomment{EMA update of counts}

$\mb_g \leftarrow \gamma \mb_g + (1-\gamma)\mb_g^{\mathrm{batch}}$
\nllabel{lnn:ema_2}
\nicecomment{EMA update of sums}
\BlankLine

\If{$n_g < \tau$}{
    $u \sim \mathrm{Unif}([S])$
    \nicecomment{sample replacement token}
    
    $\mb_g \leftarrow \hb_u$
    \nllabel{lnn:replacement}
    
    $n_g \leftarrow 1$
  }
  $\cb_g \leftarrow \mb_g / n_g$
  \nllabel{lnn:project_2}
  \nicecomment{update codeword}
}

\BlankLine

\Return{$\Cb, n, \Mb$}
\end{algorithm}
\end{minipage}
\end{figure}

\paragraph{Representation and Centroid Training.} The internal representations $\Hb$ and the expert centroids $\Wb$, in contrast to the codebook, obtain a gradient signal from the downstream prediction loss. In contrast to common vector quantization approaches for machine learning~(e.g., \citet{van2017neural,esser2021taming,razavi2019generating}), \ourmethod{} does \textit{not} use the straight-through estimation trick~\citep{bengio2013estimating}. Instead, all gradients come directly from the downstream loss and do not pass the quantization stage, as shown in~\cref{fig:overview}.

\subsection{When does the shortlist preserve routing mass?} \label{sec:theory}

We analyze the deterministic router used for validation and inference. In these phases, routing jitter is disabled and shortlists are computed from the current codebook and expert centroids. Instead of capturing the exact top-$K$ experts per token, we instead strive for selecting experts whose output probabilities are large. To this end, per token with embedding $\hb$, we define the full routing distribution
\[
\pi_e(\hb)
\;\coloneqq\;
\frac{\exp(z_e(\hb))}{\sum_{j=1}^E \exp(z_j(\hb))},
\qquad e \in [E],
\]
and the retained routing mass
\[
\mathrm{MassRecall}(\hb)
\; \coloneqq \;
\sum_{e \in L(c(\hb))} \pi_e(\hb).
\]

\begin{centralequation}
\begin{proposition}[Routing Mass Preservation]
\label{prop:mass_preservation}
Under the centroid normalization in \cref{alg:cf_router}, let $\epsilon(\hb) \coloneqq D(\hb, c(\hb))$ and let $L(c(\hb))$ denote the top-$M$ experts under $\langle c(\hb), \wb_e \rangle$. Then, during validation, 
\[
\mathrm{MassRecall}(\hb)
\;\ge\;
\exp(-2\epsilon(\hb))\, \rho_M(c(\hb)),
\]
where
\[
\rho_M(c(\hb))
\;\coloneqq\;
\sum_{e \in L(c(\hb))} \pi_e(c(\hb))
\]
is the routing mass captured by the codeword shortlist.
\end{proposition}
\end{centralequation}
The bound, \revised{which is proved in~\cref{appx:proof}}, highlights two factors: (i) the quantization error $\epsilon(\hb)$, which should be small for accurate codebooks, and (ii) the shortlist mass $\rho_M(c(\hb))$, which increases with $M$. Hence, \ourmethod{} retains most routing mass when tokens are well quantized and the shortlist is sufficiently large.

\section{Related Work} \label{sec:related_work}

\paragraph{Granular MoE.} Prior work uses fixed, non-learned hash-table routers to improve routing efficiency~\citep{roller2021hash, dos2024memory}, although such approaches are fundamentally limited both empirically and theoretically~\citep{clark2022unified, dikkala2023benefits}. More recently,~\citet{he2024mixture} proposed a learnable router based on product-key retrieval~\citep{lample2019large}, which decomposes expert centroids into two separately evaluated parts and reduces routing cost from $\Ocal(SEd)$ to $\Ocal(Sd(\sqrt{E}+K^2))$. Like our method, it uses two-stage retrieval to score experts. Unlike \ourmethod{}, it retrieves the exact top-$K$ experts rather than approximating them, but requires expert centroids $\Wb$ to be Cartesian products of two prototypes, restricting the space of admissible expert representations. The unrestricted centroids of \ourmethod{} can therefore model token--expert dependencies more flexibly. Moreover, the method of~\citet{he2024mixture} scales quadratically with the number of active experts $K$, requiring a multi-head MoE architecture with smaller per-head $K$. In contrast, \ourmethod{} requires no such architectural modifications and can be used more easily as a drop-in replacement in existing MoEs.

\paragraph{Grouping in MoE.} Our method belongs to a broader class of approaches that group experts using either overlapping~\citep{su2024maskmoe, tang2025pangu} or non-overlapping groups~\citep{jordan1991hierarchies, jordan1994hierarchical, shazeer2017outrageously}. Like \ourmethod{},~\citet{su2024maskmoe, tang2025pangu, xie2023moec} construct expert groups, but do so to mitigate load imbalance~\citep{su2024maskmoe, tang2025pangu} or ``overfitting''~\citep{xie2023moec}, rather than to reduce computation. Hierarchical MoEs~\citep{jordan1991hierarchies, jordan1994hierarchical, shazeer2017outrageously} also use non-overlapping groups to reduce computation. Unlike \ourmethod{}, however, they partition experts into fixed subgroups and restrict routing to experts within the selected groups. This avoids recomputing codebook--expert scores~(Alg.~\ref{alg:cf_router}, line~\ref{ln:codeword_expert}) but limits expressivity by preventing experts from being reused globally. Furthermore, hierarchical MoEs typically learn group selection through additive group and expert scores, whereas \ourmethod{} forms groups through vector quantization.

\paragraph{Vector Quantization in MoE.} \citet{do2024role} also use vector quantization to encode token vectors in MoEs. Unlike \ourmethod{}, however, they use codewords directly to score experts rather than to construct a retrieval index. Their method therefore provides no computational benefit.

\paragraph{Inverted File Indices \& Differentiable Retrieval.} Our method is closely related to inverted file indices~(IVF;~\citet{salton1983extended,sivic2003video,johnson2019billion}), which perform approximate nearest-neighbor search through coarse-to-fine retrieval. A coarse quantizer partitions the space, restricting each query to a small number of cells before candidates within those cells are scored exactly. \ourmethod{} applies the same principle to MoE routing: learned shortlist embeddings serve as coarse centroids, and each induces a candidate shortlist whose experts are then re-scored exactly. Unlike classical IVF systems, which are typically constructed offline for a fixed database, our method is trained end-to-end and jointly adapts its quantization and expert representations. Recent work similarly learns index representations jointly with retrieval models, blurring the distinction between indexing and retrieval~\citep{wang2022neural,zhuang2022bridging}. Our method advances this direction by enabling FLOP-efficient approximate MIPS in a gradient-compatible manner without restricting the expert space of MoE architectures.

\section{Experiments} \label{sec:experiments}

\begin{table}[t]
    \centering
    \caption{\textbf{Main results.} Perplexity (PPL) and improvement compared to the best non-\ourmethod{} baseline across different model sizes and datasets. \underline{Best results are highlighted in bold.} \ourmethod{} consistently achieves favorable perplexity compared to existing granular MoE approaches, as well as the standard coarse approach. More metrics are shown in~\cref{appx:additional_metrics}}
    \vspace{0.4em}
    \label{tab:main_results}

\resizebox{\textwidth}{!}{
\begin{tabular}{llcccccc}
\toprule
 &  & \multicolumn{2}{c}{WikiText-103}
 & \multicolumn{2}{c}{C5}
 & \multicolumn{2}{c}{OpenWebText2} \\
\cmidrule(lr){3-4}
\cmidrule(lr){5-6}
\cmidrule(lr){7-8}
Size & Method
& PPL $\downarrow$
& \makecell{Gain vs.\ best\\baseline $\uparrow$}
& PPL $\downarrow$
& \makecell{Gain vs.\ best\\baseline $\uparrow$}
& PPL $\downarrow$
& \makecell{Gain vs.\ best\\baseline $\uparrow$} \\
\midrule

\rowcolor{airrow}
\cellcolor{white}\multirow{4}{*}{Small}
& \textcolor{airtext}{\underline{AIR}}
& \textbf{19.79} & \textbf{+2.1\%}
& \textbf{546.87} & \textbf{+4.2\%}
& \textbf{49.16} & \textbf{+5.4\%} \\
& PEER & 20.30 & -0.4\% & 570.88 & +0.0\% & 51.95 & +0.0\% \\
& Hierarchical & 20.21 & +0.0\% & 918.47 & -60.9\% & 67.21 & -29.4\% \\
& Std.\ Coarse & 21.79 & -7.8\% & 796.31 & -39.5\% & 56.85 & -9.4\% \\

\midrule

\rowcolor{airrow}
\cellcolor{white}\multirow{4}{*}{Medium}
& \textcolor{airtext}{\underline{AIR}}
& \textbf{18.92} & \textbf{+3.7\%}
& \textbf{43.85} & \textbf{+4.4\%}
& \textbf{20.64} & \textbf{+3.9\%} \\
& PEER & 19.64 & +0.0\% & 45.88 & +0.0\% & 21.49 & +0.0\% \\
& Hierarchical & 24.05 & -22.4\% & 51.45 & -12.1\% & 25.68 & -19.5\% \\
& Std.\ Coarse & 21.02 & -7.0\% & 49.81 & -8.6\% & 22.77 & -5.9\% \\

\midrule

\rowcolor{airrow}
\cellcolor{white}\multirow{4}{*}{Large}
& \textcolor{airtext}{\underline{AIR}}
& -- & --
& \textbf{46.84} & \textbf{+2.1\%}
& \textbf{18.03} & \textbf{+0.7\%} \\
& PEER & -- & -- & 47.86 & +0.0\% & 18.16 & +0.0\% \\
& Hierarchical & -- & -- & 51.73 & -8.1\% & 20.24 & -11.5\% \\
& Std.\ Coarse & -- & -- & 49.06 & -2.5\% & 18.67 & -2.8\% \\

\bottomrule
\end{tabular}
}
\end{table}

We aim to answer the following questions: \textbf{(1)} How does \ourmethod{} compare with existing techniques for granular MoE? \textbf{(2)} How does \ourmethod{} compare with a coarse MoE architecture? \textbf{(3)} How well does \ourmethod{} scale in the amount of experts? \textbf{(4)} How do the different components of \ourmethod{} affect performance?
 
We address questions \textbf{(1)} to \textbf{(3)} in the main text and refer to \cref{appx:ablations} for ablation experiments (\textbf{(4)}). 

\subsection{Setup}

We describe the core setup here and defer details to~\cref{appx:setup}.

\paragraph{Datasets.}
We evaluate on WikiText-103~\citep{merity2016pointer}, C5~\citep{commoncrawl}, and OpenWebText2~\citep{biderman2022datasheet}.

\paragraph{Baselines.}
(1) Coarse MoE with standard routing and $K=1$ (matched active parameters) and expert-choice gating, 
(2) granular MoE with standard routing (only shown in~\cref{sec:scale_num_experts} and~\cref{appx:beyond_flops} because of memory overflow), 
(3) granular PEER routing~\citep{he2024mixture}, 
(4) granular hierarchical MoE~\citep{shazeer2017outrageously}.

\paragraph{Metrics.}
Perplexity (PPL; $\downarrow$); entropy over expert usage $-\EE_{\hat{p}}[\log \hat{p}]$ (Entropy); fraction of dead experts, i.e., experts that do not receive a single token (Dead Experts); overlap fraction between top-$K$ estimate and true top-$K$ (Overlap; only for \ourmethod{}).

\paragraph{Architecture.}
We use the Llama3 transformer~\citep{grattafiori2024llama}, replacing the middle FFN with a MoE layer that consists of tiny experts with a single intermediate dimension~\citep{he2024mixture} and adding post-MoE layer normalization~\citep{ba2016layer,krajewski2024scaling}. We report results for small ($512^2$ experts; $\approx 313\mathrm{M}$ total parameters), medium ($1024^2$ experts; $\approx 2\mathrm{B}$ parameters total), and large ($1024^2$ experts; $\approx 2.9\mathrm{B}$ parameters total) models. For WikiText-103, a relatively small data set, we omit the large model due to strong overfitting.

\paragraph{Training.}
We use standard pre-training hyperparameters (see~\cref{appx:hyperparameters}) and align baselines by active parameters (details in~\cref{appx:baselines}). For WikiText-103, we train for 10 (small) and 3 (medium) epochs, respectively. For C5 and OpenWebText2, we use $10\times$ the amount of the MoE-free (that is, without the MoE layer) model parameters in tokens.

\begin{figure}[t]
    \centering
    \includegraphics[width=0.8\textwidth]{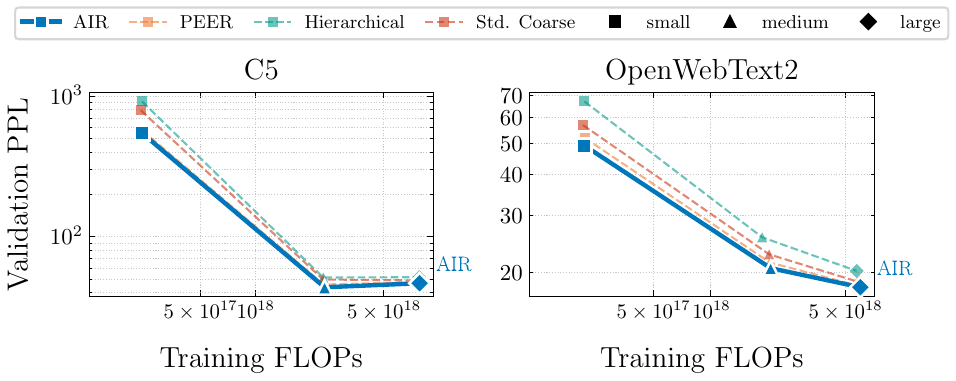}
    \caption{\textbf{Performance–efficiency trade-off.} Validation perplexity vs. training FLOPs.~\ourmethod{} (blue) occupies the lower-left region of the plot and thus achieves the strongest Pareto trade-off.}
    \label{fig:performance_vs_flops}
\end{figure}

\subsection{Main Results}

The main results, shown in~\cref{tab:main_results} and~\cref{fig:performance_vs_flops}, demonstrate the merits of the proposed routing architecture:~\ourmethod{} yields consistent gains over the considered MoE baselines across datasets and model sizes (multi-seed results are shown in~\cref{appx:multiseed_runs}). Compared to hierarchical and PEER routers, \ourmethod{} achieves persistent PPL reductions of up to $5.4\%$ relative to the best baseline~(\cref{tab:main_results}). Due to the vast amount of expert parameters, even the routing cost of the standard coarse approach becomes non-negligible and therefore does not yield large FLOP benefits, while performing worse than both PEER and \ourmethod{} in terms of perplexity~(\cref{fig:performance_vs_flops}). The hierarchical router generally yields the weakest results, typically underperforming even the coarse-grained MoE baseline. We hypothesize that this poor performance results from the absence of overlap between groups: although some experts may be broadly useful across many tokens, they can only be selected when their corresponding coarse group is activated. In contrast, the shortlists constructed by \ourmethod{} naturally overlap, and some experts are included substantially more often than others, as shown empirically in~\cref{appx:overlap_analysis}.

\subsection{Router Efficiency at Scale} \label{sec:scale_num_experts}

\begin{figure}[t]
    \centering
    \includegraphics[width=1\textwidth]{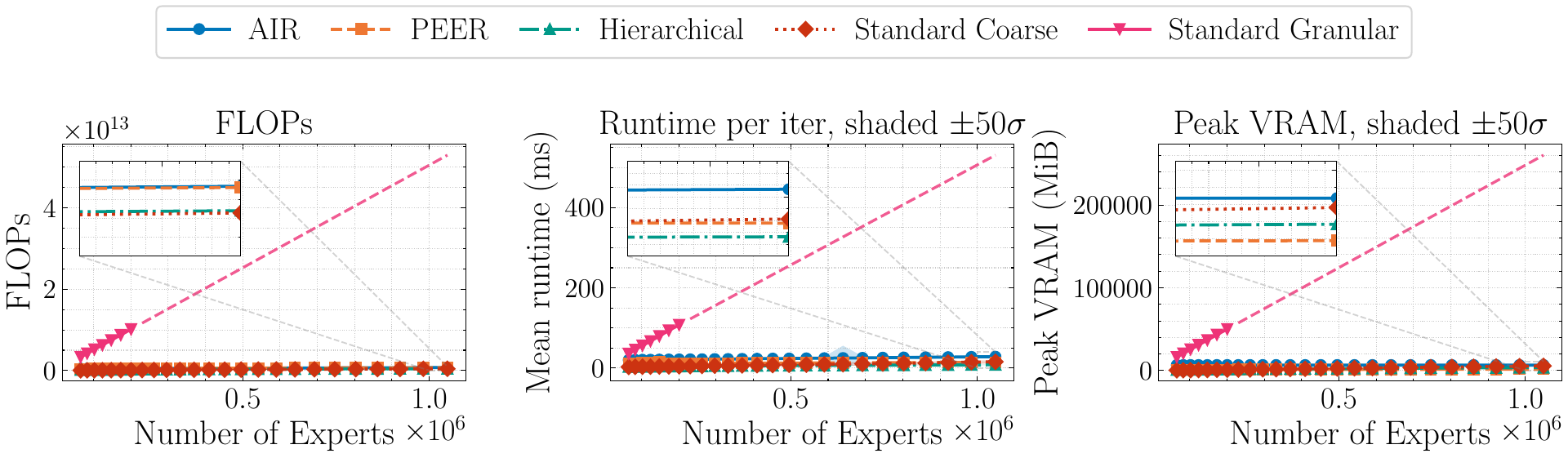}
    \caption{\revised{\textbf{Router efficiency.} Mean results $\pm 50$ standard deviations across $100$ iterations. The dashed extension of the standard granular approach (pink) is a linear extrapolation because of memory overflow. Standard granular routing leads to substantial computational and memory overhead as the amount of experts is scaled up. \ourmethod{}, in contrast, keeps resource demand at a similar level to existing efficient granular approaches.}}
    \label{fig:router_comparisons_experts}
\end{figure}

\revised{In order to assess router efficiency beyond FLOPs, we run an additional experiment for a combined forward and backward pass of the routers in isolation. We compare router FLOPs, runtime per iteration, and peak VRAM usage as a function of the number of experts (details regarding setup are shown in~\cref{appx:scaling_setup} and further scaling experiments are shown in~\cref{appx:beyond_flops}).}

\revised{The results, shown in~\cref{fig:router_comparisons_experts}, demonstrate that \ourmethod{} (blue) scales substantially better than the standard granular router (pink) and comparably to the hierarchical and PEER routers across all reported metrics.}

\section{Discussion \& Limitations} \label{sec:discussion}

\paragraph{Shortlist Updates.} A distinguishing feature of \ourmethod{} is that, unlike prior approaches, it updates all shortlists after each optimizer step~(\cref{sec:forward_pass}). This incurs a computational cost of $\Ocal(EGd)$ per update, where $G$ denotes the codebook size and $E$ the number of experts. In practice, this overhead is typically tolerable because it is independent of the effective batch size (that is, the micro batch size multiplied by the amount of gradient accumulation steps). Nevertheless, the resulting compute savings depend strongly on $G$, $E$, and the effective batch size $S$ (see~\cref{appx:beyond_flops,appx:codebook_size_vs_shortlist_size}).

\revised{\paragraph{Expert Parallelism.} A key open question for making granular MoEs viable for large-scale foundation-model training is how to design efficient expert-parallel systems, such as DeepSpeed-MoE~\citep{rajbhandari2022deepspeed} and Megatron-LM~\citep{shoeybi2019megatron}. In the granular regime, expert parallelism can substantially increase inter-device memory traffic and degrade runtime. Although we do not study this direction, architectures such as \ourmethod{} may be well suited to distributed training because expert shortlists create routing locality. If each shortlist is assigned to one device, its expert centroids and FFN weights can be colocated there, requiring each token to be routed only to the device hosting its shortlist and potentially reducing cross-device communication.}

\revised{\paragraph{Custom Kernel Implementations.}
All implementations used to obtain the runtime and peak-VRAM results reported in~\cref{sec:scale_num_experts} and~\cref{appx:beyond_flops}, including both \ourmethod{} and the considered baselines, rely exclusively on functionality provided by PyTorch. We do not develop method-specific custom kernels, which could substantially affect the observed hardware-level performance. Designing and carefully optimizing custom kernels represents an important direction for future work.}

\section{Conclusion} \label{sec:conclusion}

We introduced \ourmethod{}, an inverted-index-inspired routing architecture for granular mixture-of-experts. By using a vector-quantization-based coarse-to-fine scheme, \ourmethod{} approximates top-$K$ routing without fixed grouping or imposing structural constraints on expert representations. We further proposed a bi-level training strategy for learning the adaptive codebook and provided a lower bound on the retained routing mass. Empirically, \ourmethod{} achieves a favorable perplexity--FLOPs trade-off in granular MoE settings, comparing favorably to existing MoE approaches while remaining substantially cheaper than exact granular routing. Overall, our results suggest that inverted-index structures provide an effective mechanism for routing in sparse and highly granular mixture-of-experts architectures.

\subsection*{Reproducibility statement}

All source code, including instructions in the \texttt{README.md}, is available at \href{https://anonymous.4open.science/r/adaptive_inverted_index_routing_code-6C06}{https://anonymous.4open.science/r/adaptive\_inverted\_index\_routing\_code-D6FE/}.

\bibliography{main}

@article{do2024role,
  title={{On the Role of Discrete Representation in Sparse Mixture of Experts}},
  author={Do, Giang and Pham, Kha and Le, Hung and Tran, Truyen},
  journal={arXiv preprint arXiv:2411.19402},
  year={2024}
}

@article{tang2025pangu,
  title={{Pangu Pro MoE: Mixture of Grouped Experts for Efficient Sparsity}},
  author={Tang, Yehui and Li, Xiaosong and Liu, Fangcheng and Guo, Wei and Zhou, Hang and Wang, Yaoyuan and Han, Kai and Yu, Xianzhi and Li, Jinpeng and Zang, Hui and others},
  journal={arXiv preprint arXiv:2505.21411},
  year={2025}
}

@article{su2024maskmoe,
  title={{MaskMoE: Boosting Token-Level Learning via Routing Mask in Mixture-of-Experts}},
  author={Su, Zhenpeng and Lin, Zijia and Bai, Xue and Wu, Xing and Xiong, Yizhe and Lian, Haoran and Ma, Guangyuan and Chen, Hui and Ding, Guiguang and Zhou, Wei and others},
  journal={arXiv preprint arXiv:2407.09816},
  year={2024}
}

@article{zhou2022mixture,
  title={{Mixture-of-Experts with Expert Choice Routing}},
  author={Zhou, Yanqi and Lei, Tao and Liu, Hanxiao and Du, Nan and Huang, Yanping and Zhao, Vincent and Dai, Andrew M and Le, Quoc V and Laudon, James and others},
  journal={Advances in Neural Information Processing Systems},
  volume={35},
  pages={7103--7114},
  year={2022}
}

@article{ba2016layer,
  title={{Layer Normalization}},
  author={Ba, Jimmy Lei and Kiros, Jamie Ryan and Hinton, Geoffrey E.},
  journal={arXiv preprint arXiv:1607.06450},
  year={2016}
}

@article{krajewski2024scaling,
  title={{Scaling Laws for Fine-Grained Mixture of Experts}},
  author={Krajewski, Jakub and Ludziejewski, Jan and Adamczewski, Kamil and Pi{\'o}ro, Maciej and Krutul, Micha{\l} and Antoniak, Szymon and Ciebiera, Kamil and Kr{\'o}l, Krystian and Odrzyg{\'o}{\'z}d{\'z}, Tomasz and Sankowski, Piotr and others},
  journal={International Conference on Machine Learning},
  year={2024}
}

@article{shazeer2020glu,
  title={{GLU Variants Improve Transformer}},
  author={Shazeer, Noam},
  journal={arXiv preprint arXiv:2002.05202},
  year={2020}
}

@article{williams2020hierarchical,
  title={{Hierarchical Quantized Autoencoders}},
  author={Williams, Will and Ringer, Sam and Ash, Tom and MacLeod, David and Dougherty, Jamie and Hughes, John},
  journal={Advances in Neural Information Processing Systems},
  volume={33},
  pages={4524--4535},
  year={2020}
}

@article{wang2022neural,
  title={{A Neural Corpus Indexer for Document Retrieval}},
  author={Wang, Yujing and Hou, Yingyan and Wang, Haonan and Miao, Ziming and Wu, Shibin and Chen, Qi and Xia, Yuqing and Chi, Chengmin and Zhao, Guoshuai and Liu, Zheng and others},
  journal={Advances in Neural Information Processing Systems},
  volume={35},
  pages={25600--25614},
  year={2022}
}

@article{zhuang2022bridging,
  title={{Bridging the Gap Between Indexing and Retrieval for Differentiable Search Index with Query Generation}},
  author={Zhuang, Shengyao and Ren, Houxing and Shou, Linjun and Pei, Jian and Gong, Ming and Zuccon, Guido and Jiang, Daxin},
  journal={arXiv preprint arXiv:2206.10128},
  year={2022}
}

@article{johnson2019billion,
  title={{Billion-scale similarity search with GPUs}},
  author={Johnson, Jeff and Douze, Matthijs and J{\'e}gou, Herv{\'e}},
  journal={IEEE Transactions on Big Data},
  volume={7},
  number={3},
  pages={535--547},
  year={2019},
  publisher={IEEE}
}

@article{sivic2003video,
  title={{Video Google: A text retrieval approach to object matching in videos}},
  author={Sivic and Zisserman},
  journal={International Conference on Computer Vision},
  pages={1470--1477},
  year={2003},
  organization={IEEE}
}

@article{raffel2020exploring,
  title={{Exploring the Limits of Transfer Learning with a Unified Text-to-Text Transformer}},
  author={Raffel, Colin and Shazeer, Noam and Roberts, Adam and Lee, Katherine and Narang, Sharan and Matena, Michael and Zhou, Yanqi and Li, Wei and Liu, Peter J.},
  journal={Journal of Machine Learning Research},
  volume={21},
  number={140},
  pages={1--67},
  year={2020}
}

@article{grattafiori2024llama,
  title={{The Llama 3 Herd of Models}},
  author={Grattafiori, Aaron and Dubey, Abhimanyu and Jauhri, Abhinav and Pandey, Abhinav and Kadian, Abhishek and Al-Dahle, Ahmad and Letman, Aiesha and Mathur, Akhil and Schelten, Alan and Vaughan, Alex and others},
  journal={arXiv preprint arXiv:2407.21783},
  year={2024}
}

@article{esser2021taming,
  title={{Taming Transformers for High-Resolution Image Synthesis}},
  author={Esser, Patrick and Rombach, Robin and Ommer, Björn},
  journal={Conference on Computer Vision and Pattern Recognition},
  pages={12873--12883},
  year={2021}
}

@article{gray1984vector,
  title={{Vector Quantization}},
  author={Gray, Robert},
  journal={IEEE Assp Magazine},
  volume={1},
  number={2},
  pages={4--29},
  year={1984}
}

@article{linde1980algorithm,
  title={{An Algorithm for Vector Quantizer Design}},
  author={Linde, Yoseph and Buzo, Andres and Gray, Robert},
  journal={IEEE Transactions on communications},
  volume={28},
  number={1},
  pages={84--95},
  year={1980}
}

@article{abuzaid2019index,
  title={{To Index or Not to Index: Optimizing Exact Maximum Inner Product Search}},
  author={Abuzaid, Firas and Sethi, Geet and Bailis, Peter and Zaharia, Matei},
  journal={International Conference on Data Engineering},
  pages={1250--1261},
  year={2019}
}

@article{indyk1998approximate,
  title={{Approximate nearest neighbors: towards removing the curse of dimensionality}},
  author={Indyk, Piotr and Motwani, Rajeev},
  journal={ACM Symposium on Theory of Computing},
  pages={604--613},
  year={1998}
}

@article{mcqueen1967some,
  title={{Some Methods of Classification and Analysis of Multivariate Observations}},
  author={McQueen, James B.},
  journal={Proc. of 5th Berkeley Symposium on Math. Stat. and Prob.},
  pages={281--297},
  year={1967}
}

@article{dhillon2001concept,
  title={{Concept Decompositions for Large Sparse Text Data Using Clustering}},
  author={Dhillon, Inderjit S. and Modha, Dharmendra S.},
  journal={Machine Learning},
  volume={42},
  number={1},
  pages={143--175},
  year={2001}
}

@article{biderman2022datasheet,
  title={{Datasheet for the Pile}},
  author={Biderman, Stella and Bicheno, Kieran and Gao, Leo},
  journal={arXiv preprint arXiv:2201.07311},
  year={2022}
}

@article{bengio2013estimating,
  title={{Estimating or Propagating Gradients Through Stochastic Neurons for Conditional Computation}},
  author={Bengio, Yoshua and L{\'e}onard, Nicholas and Courville, Aaron},
  journal={arXiv preprint arXiv:1308.3432},
  year={2013}
}

@article{van2017neural,
  title={{Neural Discrete Representation Learning}},
  author={Van Den Oord, Aaron and Vinyals, Oriol and others},
  journal={Advances in Neural Information Processing Systems},
  volume={30},
  year={2017}
}

@article{salton1983extended,
  title={{Extended Boolean Information Retrieval}},
  author={Salton, Gerard and Fox, Edward A. and Wu, Harry},
  journal={Communications of the ACM},
  volume={26},
  number={11},
  pages={1022--1036},
  year={1983}
}

@article{xie2023moec,
  title={{MoEC: Mixture of Expert Clusters}},
  author={Xie, Yuan and Huang, Shaohan and Chen, Tianyu and Wei, Furu},
  journal={AAAI Conference on Artificial Intelligence},
  volume={37},
  number={11},
  pages={13807--13815},
  year={2023}
}

@article{shrivastava2014asymmetric,
  title={Asymmetric LSH (ALSH) for sublinear time maximum inner product search (MIPS)},
  author={Shrivastava, Anshumali and Li, Ping},
  journal={Advances in Neural Information Processing Systems},
  volume={27},
  year={2014}
}

@article{nguyen2024statistical,
  title={Statistical advantages of perturbing cosine router in sparse mixture of experts},
  author={Nguyen, Huy and Akbarian, Pedram and Pham, Trang and Nguyen, Trang and Zhang, Shujian and Ho, Nhat},
  journal={arXiv preprint arXiv:2405.14131},
  volume={6},
  year={2024}
}

@article{fedus2022switch,
  title={{Switch Transformers: Scaling to Trillion Parameter Models with Simple and Efficient Sparsity}},
  author={Fedus, William and Zoph, Barret and Shazeer, Noam},
  journal={Journal of Machine Learning Research},
  volume={23},
  number={120},
  pages={1--39},
  year={2022}
}

@article{jordan1991hierarchies,
  title={{Hierarchies of Adaptive Experts}},
  author={Jordan, Michael I. and Jacobs, Robert A.},
  journal={Advances in Neural Information Processing Systems},
  volume={4},
  year={1991}
}

@article{jordan1994hierarchical,
  title={{Hierarchical Mixtures of Experts and the EM Algorithm}},
  author={Jordan, Michael I. and Jacobs, Robert A.},
  journal={Neural Computation},
  volume={6},
  number={2},
  pages={181--214},
  year={1994},
  publisher={MIT Press}
}

@article{team2025kimi,
  title={{Kimi-VL Technical Report}},
  author={Du, Angang and Yin, Bohong and Xing, Bowei and Qu, Bowen and Wang, Bowen and Chen, Cheng and Zhang, Chenlin and Du, Chenzhuang and Wei, Chu and others},
  journal={arXiv preprint arXiv:2504.07491},
  year={2025}
}

@article{rajbhandari2022deepspeed,
  title={{DeepSpeed-MoE: Advancing Mixture-of-Experts Inference and Training to Power Next-Generation AI Scale}},
  author={Rajbhandari, Samyam and Li, Conglong and Yao, Zhewei and Zhang, Minjia and Aminabadi, Reza Yazdani and Awan, Ammar Ahmad and Rasley, Jeff and He, Yuxiong},
  booktitle={International Conference on Machine Learning},
  pages={18332--18346},
  year={2022}
}

@article{li2024merge,
  title={{Merge, Then Compress: Demystify Efficient SMoE with Hints from Its Routing Policy}},
  author={Li, Pingzhi and Zhang, Zhenyu and Yadav, Prateek and Sung, Yi-Lin and Cheng, Yu and Bansal, Mohit and Chen, Tianlong},
  journal={International Conference on Learning Representations},
  year={2024}
}

@article{li2024quantmoe,
  title={{QuantMoE-Bench: Examining Post-Training Quantization for Mixture-of-Experts}},
  author={Li, Pingzhi and Jin, Xiaolong and Tan, Zhen and Cheng, Yu and Chen, Tianlong},
  journal={arXiv preprint arXiv:2406.08155},
  year={2024}
}

@article{shoeybi2019megatron,
  title={{Megatron-LM: Training Multi-Billion Parameter Language Models Using Model Parallelism}},
  author={Shoeybi, Mohammad and Patwary, Mostofa and Puri, Raul and LeGresley, Patrick and Casper, Jared and Catanzaro, Bryan},
  journal={arXiv preprint arXiv:1909.08053},
  year={2019}
}

@article{su2024roformer,
  title={{RoFormer: Enhanced transformer with Rotary Position Embedding}},
  author={Su, Jianlin and Ahmed, Murtadha and Lu, Yu and Pan, Shengfeng and Bo, Wen and Liu, Yunfeng},
  journal={Neurocomputing},
  volume={568},
  pages={127063},
  year={2024}
}

@article{ainslie2023gqa,
  title={{GQA: Training Generalized Multi-Query Transformer Models from Multi-Head Checkpoints}},
  author={Ainslie, Joshua and Lee-Thorp, James and De Jong, Michiel and Zemlyanskiy, Yury and Lebr{\'o}n, Federico and Sanghai, Sumit},
  journal={Empirical Methods in Natural Language Processing},
  pages={4895--4901},
  year={2023}
}

@article{razavi2019generating,
  title={{Generating Diverse High-Fidelity Images with VQ-VAE-2}},
  author={Razavi, Ali and Van den Oord, Aaron and Vinyals, Oriol},
  journal={Advances in Neural Information Processing Systems},
  volume={32},
  year={2019}
}

@article{he2024mixture,
  title={{Mixture of a Million Experts}},
  author={He, Xu Owen},
  journal={arXiv preprint arXiv:2407.04153},
  year={2024}
}

@article{jiang2024mixtral,
  title={{Mixtral of Experts}},
  author={Jiang, Albert Q. and Sablayrolles, Alexandre and Roux, Antoine and Mensch, Arthur and Savary, Blanche and Bamford, Chris and Chaplot, Devendra Singh and Casas, Diego de las and Hanna, Emma Bou and Bressand, Florian and others},
  journal={arXiv preprint arXiv:2401.04088},
  year={2024}
}

@article{vaswani2017attention,
  title={{Attention Is All You Need}},
  author={Vaswani, Ashish and Shazeer, Noam and Parmar, Niki and Uszkoreit, Jakob and Jones, Llion and Gomez, Aidan N and Kaiser, {\L}ukasz and Polosukhin, Illia},
  journal={Advances in Neural Information Processing Systems},
  volume={30},
  year={2017}
}

@article{jacobs1991adaptive,
  title={{Adaptive Mixtures of Local Experts}},
  author={Jacobs, Robert A. and Jordan, Michael I. and Nowlan, Steven J. and Hinton, Geoffrey E.},
  journal={Neural Computation},
  volume={3},
  number={1},
  pages={79--87},
  year={1991}
}

@article{lample2019large,
  title={{Large Memory Layers with Product Keys}},
  author={Lample, Guillaume and Sablayrolles, Alexandre and Ranzato, Marc'Aurelio and Denoyer, Ludovic and J{\'e}gou, Herv{\'e}},
  journal={Advances in Neural Information Processing Systems},
  volume={32},
  year={2019}
}

@article{du2022glam,
  title={{GLaM: Efficient Scaling of Language Models with Mixture-of-Experts}},
  author={Du, Nan and Huang, Yanping and Dai, Andrew M and Tong, Simon and Lepikhin, Dmitry and Xu, Yuanzhong and Krikun, Maxim and Zhou, Yanqi and Yu, Adams Wei and Firat, Orhan and others},
  journal={International Conference on Machine Learning},
  pages={5547--5569},
  year={2022}
}

@article{clark2022unified,
  title={{Unified Scaling Laws for Routed Language Models}},
  author={Clark, Aidan and de Las Casas, Diego and Guy, Aurelia and Mensch, Arthur and Paganini, Michela and Hoffmann, Jordan and Damoc, Bogdan and Hechtman, Blake and Cai, Trevor and Borgeaud, Sebastian and others},
  journal={International Conference on Machine Learning},
  pages={4057--4086},
  year={2022}
}

@article{dos2024memory,
  title={{Memory Augmented Language Models through Mixture of Word Experts}},
  author={Nogueira Dos Santos, Cicero and Lee-Thorp, James and Noble, Isaac and Chang, Chung-Ching and Uthus, David},
  journal={North American Chapter of the Association for Computational Linguistics: Human Language Technologies},
  pages={4425--4438},
  year={2024}
}

@article{roller2021hash,
  title={{Hash Layers For Large Sparse Models}},
  author={Roller, Stephen and Sukhbaatar, Sainbayar and Weston, Jason},
  journal={Advances in Neural Information Processing Systems},
  volume={34},
  pages={17555--17566},
  year={2021}
}

@article{dikkala2023benefits,
  title={{On the Benefits of Learning to Route in Mixture-of-Experts Models}},
  author={Dikkala, Nishanth and Ghosh, Nikhil and Meka, Raghu and Panigrahy, Rina and Vyas, Nikhil and Wang, Xin},
  journal={Conference on Empirical Methods in Natural Language Processing},
  pages={9376--9396},
  year={2023}
}

@article{shazeer2017outrageously,
  title={{Outrageously Large Neural Networks: The Sparsely-Gated Mixture-of-Experts Layer}},
  author={Shazeer, Noam and Mirhoseini, Azalia and Maziarz, Krzysztof and Davis, Andy and Le, Quoc and Hinton, Geoffrey and Dean, Jeff},
  journal={International Conference on Learning Representations},
  year={2017}
}

@article{merity2016pointer,
  title={{Pointer Sentinel Mixture Models}},
  author={Stephen Merity and Caiming Xiong and James Bradbury and Richard Socher},
  year={2017},
  journal={International Conference on Learning Representations}
}

@article{commoncrawl,
  title   = {Common Crawl Creative Commons Corpus (C5)},
  author  = {{Common Crawl Foundation}},
  year    = {2025},
  journal = {Common Crawl},
  note    = {Available at https://huggingface.co/datasets/BramVanroy/CommonCrawl-CreativeCommons, accessed 2025-08-11}
}
\bibliographystyle{tmlr}

\newpage

\onecolumn
\appendix
\addcontentsline{toc}{section}{Appendix}
\vspace*{\fill}
{ \centering\part{{\huge{Appendix}}} \parttoc }
\vspace*{\fill}
\newpage

\section{Full Optimization Loop} \label{appx:optimization_loop}

\begin{figure}
    \centering
    \includegraphics[width=0.35\textwidth]{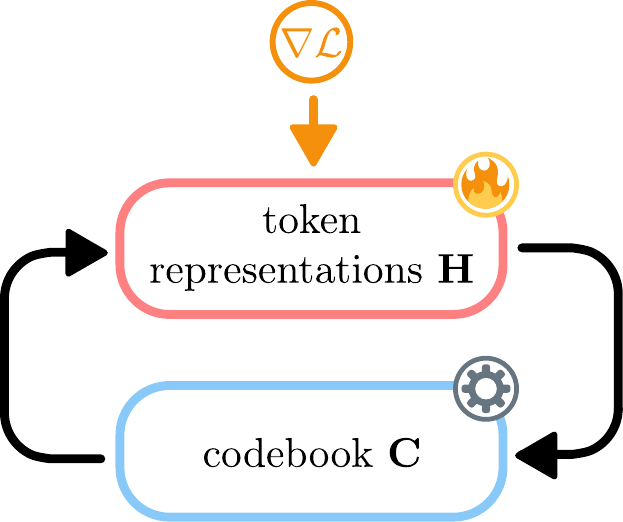}
    \caption{\textbf{Feedback Loop}. \ourmethod{} forms a bi-level optimization loop between gradient-learned token representations and a codebook updated via adaptive clustering.}
    \label{fig:feedback_loop}
\end{figure}

\Cref{alg:full_optimization_loop} summarizes the full training procedure of
\ourmethod{}. The algorithm can be viewed as a bi-level optimization loop between
gradient-trained model parameters and a gradient-free adaptive codebook. Let
$\theta$ denote all parameters trained by backpropagation, including the transformer
parameters, expert FFN parameters, and expert centroids $\Wb$. In contrast, the
codebook $\Cb$ and its exponential-moving-average statistics $(n,\Mb)$ are
auxiliary router state and are not updated by the optimizer.

Importantly, the codeword-to-expert shortlists are not recomputed after every
codebook update. Instead, a cached shortlist table is reused throughout each
gradient-accumulation window. The codebook therefore continues to adapt across
microbatches while the shortlist cache remains fixed and may become slightly
stale. This deliberate approximation amortizes the
$\Ocal(EGd)$ shortlist-construction cost over the effective batch. After each
optimizer step, the cache is invalidated and lazily reconstructed during the
subsequent forward pass using the updated expert centroids and codebook.

\begin{figure}[t]
\centering
\begin{minipage}{0.945\textwidth}
\begin{algorithm}[H] 
\DontPrintSemicolon
\LinesNumbered
\caption{Full Optimization Loop for \ourmethod{}}
\label{alg:full_optimization_loop}

\KwIn{
Training data $\Dcal$; initial parameters $\theta_0$; initial codebook $\Cb_0$;
EMA counts $n_0$; EMA sums $\Mb_0$; optimizer $\texttt{optim}$; training horizon
$T$; gradient accumulation steps $A$; shortlist size $M$; top-$K$
}
\KwOut{Trained parameters $\theta_T$ and codebook $\Cb_T$}

\BlankLine

Initialize empty shortlist cache $\{L_g\}_{g=1}^G \leftarrow \varnothing$\;

\BlankLine

\For{$t \leftarrow 0,1,\dots,T-1$}{
    $\texttt{optim.zero\_grad}()$\;

    \BlankLine

    \For{$a \leftarrow 1,\dots,A$}{
        $\Dcal_{\mathrm{b}} \sim \Dcal$
        \nicecomment{sample microbatch}

        \BlankLine
        Compute token representations
        $\Hb_t = \{\hb_s\}_{s=1}^{S}$
        using parameters $\theta_t$\;

        \BlankLine
        \nicefullcomment{gradient-free codebook update}\;
        $(\Cb_t,n_t,\Mb_t,\{g_s\}_{s=1}^{S})
        \leftarrow
        \textsc{AdaptiveKMeans}(\Hb_t,\Cb_t,n_t,\Mb_t)$ \;
        \nicefullcomment{codeword assignment} \;
        \For{$s \leftarrow 1$ \KwTo $S$}{
          $g_s \leftarrow \arg\min_{g \in [G]} D (\hb_s, \cb_g)$
        }

        \BlankLine
        \nicefullcomment{normalize expert centroids for routing}\;
        $\wb_e \leftarrow \Pi_{\mathbb{S}^d}(\wb_e)$,
        \quad $\forall e\in[E]$\;

        \BlankLine
        \nicefullcomment{lazily construct codeword-to-expert shortlists}\;
        \If{shortlist cache is empty}{
            \For{$g \leftarrow 1$ \KwTo $G$}{
                $\etab_1 \sim \Ncal(\zerob, \sigma_{\mathrm{jit}}^2 \Ib)$ \;
                $L_g
                \leftarrow
                \mathrm{TopM}_{e\in[E]}
                (\langle \cb_g,\wb_e\rangle + \etab_1)$
                \nicecomment{cache expert indices}
            }
        }

        \BlankLine
        \nicefullcomment{fine scoring within cached shortlists}\;
        \For{$s \leftarrow 1$ \KwTo $S$}{
            $\etab_2 \sim \Ncal(\zerob, \sigma_{\mathrm{jit}}^2 \Ib)$ \\
              $T_s, \zb_s \leftarrow \mathrm{TopK}_{e \in L_{g_s}}\big(\langle \hb_s, \wb_e \rangle + \etab_2 \big)$
            
        }

        \BlankLine
        Compute prediction loss $\Lcal_{\mathrm{LM}}$\;
        Compute load balancing $\Lcal_{\mathrm{router}}$\;
        $\Lcal
        \leftarrow
        \Lcal_{\mathrm{LM}}
        +
        \Lcal_{\mathrm{router}}$\;

        Backpropagate $\Lcal/A$
        \nicecomment{gradients update $\theta$, not $\Cb_t$}
    }

    \BlankLine
    \nicefullcomment{gradient update of model parameters}\;
    $\theta_{t+1} \leftarrow \texttt{optim.step}(\theta_t)$\;

    \BlankLine
    \nicefullcomment{invalidate shortlist cache after optimizer step}\;
    $\{L_g\}_{g=1}^{G} \leftarrow \varnothing$\;
}

\BlankLine
\Return{$\theta_T,\Cb_T$}
\end{algorithm}
\end{minipage}
\end{figure}

\section{Proof for \texorpdfstring{\cref{prop:mass_preservation}}{Proposition 1}} \label{appx:proof}

\begin{proof}
Let $\cb = c(\hb)$ denote the assigned codeword and define
\[
\epsilon(\hb) \;\coloneqq\; D(\hb, \cb).
\]
Further, let
\[
L(\cb) \;=\; \mathrm{TopM}_{e\in[E]} \langle \cb,\wb_e\rangle
\]
denote the shortlist induced by the codeword. By definition,
\[
\mathrm{MassRecall}(\hb)
\;=\;
\sum_{e\in L(\cb)} \pi_e(\hb)
\;=\;
\frac{\sum_{e\in L(\cb)} \exp\{z_e(\hb)\}}
{\sum_{j=1}^E \exp\{z_j(\hb)\}}.
\]

We first compare token and codeword logits. By the Cauchy-Schwarz inequality, for any expert $e\in[E]$,
\[
|z_e(\hb)-z_e(\cb)|
\;=\;
|\langle \wb_e,\hb-\cb\rangle|
\;\le\;
\|\wb_e\|\,\|\hb-\cb\|
\;\le\;
\epsilon(\hb),
\]
where the last step uses the assumption $\|\wb_e\|\le 1$. Hence, for every $e\in[E]$,
\[
z_e(\hb)\;\ge\; z_e(\cb)-\epsilon(\hb)
\qquad\text{and}\qquad
z_e(\hb)\;\le\; z_e(\cb)+\epsilon(\hb).
\]
Exponentiating yields
\[
\exp\{z_e(\hb)\}
\;\ge\;
\exp(-\epsilon(\hb))\,\exp\{z_e(\cb)\},
\]
and
\[
\exp\{z_e(\hb)\}
\;\le\;
\exp(\epsilon(\hb))\,\exp\{z_e(\cb)\}.
\]

Applying the lower bound in the numerator and the upper bound in the denominator gives
\[
\mathrm{MassRecall}(\hb)
=
\frac{\sum_{e\in L(\cb)} \exp\{z_e(\hb)\}}
{\sum_{j=1}^E \exp\{z_j(\hb)\}}
\;\ge\;
\frac{
\exp(-\epsilon(\hb))
\sum_{e\in L(\cb)} \exp\{z_e(\cb)\}
}{
\exp(\epsilon(\hb))
\sum_{j=1}^E \exp\{z_j(\cb)\}
}.
\]
Therefore,
\[
\mathrm{MassRecall}(\hb)
\;\ge\;
\exp(-2\epsilon(\hb))
\frac{\sum_{e\in L(\cb)} \exp\{z_e(\cb)\}}
{\sum_{j=1}^E \exp\{z_j(\cb)\}}.
\]

It remains to identify the fraction on the right-hand side. By the definition of the full routing distribution at the codeword,
\[
\pi_e(\cb)
=
\frac{\exp\{z_e(\cb)\}}
{\sum_{j=1}^E \exp\{z_j(\cb)\}},
\]
hence
\[
\frac{\sum_{e\in L(\cb)} \exp\{z_e(\cb)\}}
{\sum_{j=1}^E \exp\{z_j(\cb)\}}
=
\sum_{e\in L(\cb)} \pi_e(\cb)
=
1-\sum_{e\notin L(\cb)} \pi_e(\cb).
\]
Using the definition
\[
\rho_M(\cb)
\;\coloneqq\;
\sum_{e \in L(\cb)} \pi_e(\cb),
\]
we obtain
\[
\mathrm{MassRecall}(\hb)
\;\ge\;
\exp(-2\epsilon(\hb)) \, \rho_M(\cb),
\]
which proves the desired result.
\end{proof}

\section{Further Implementation Details} \label{appx:implementation_details}

\paragraph{Overview.}
\Cref{lst:AIRrouter_pseudocode} shows a simplified PyTorch-style implementation of the proposed AIR Router.

\paragraph{Step 1: Shortlist selection.}
Each token representation $\hb \in \mathbb{R}^H$ is assigned to a shortlist centroid using a vector-quantization module 
(\texttt{AdaptiveKmeans}). This produces one discrete shortlist index per token and updates the codebook’s exponential-moving-average statistics without backpropagating through the quantization operation.

\paragraph{Step 2: Experts per shortlist.}
For each shortlist centroid, a fixed subset of experts is selected. This is implemented by computing similarities 
between shortlist centroids and expert feature vectors, followed by a top-$M$ selection. 
This step defines a coarse shortlist of candidate experts per shortlist.

\paragraph{Step 3: Candidate scoring.}
After shortlist selection, each token is scored only against the experts contained in its selected shortlist. This avoids dense scoring against all experts and reduces the routing cost from $\mathcal{O}(SE)$ to $\mathcal{O}(SM)$. A practical challenge is that tokens are assigned to shortlists irregularly, so different shortlists may receive different numbers of tokens. To handle this efficiently, our implementation of \ourmethod{} uses PyTorch's \texttt{grouped\_mm} operation, which evaluates the per-shortlist matrix multiplications as a grouped GEMM. This allows all non-empty shortlists to be processed in parallel without materializing a token-expanded candidate tensor of shape $S \times M \times H$ and without padding each shortlist to a fixed maximum capacity.

\paragraph{Step 4: Final expert selection.}
From the candidate set, the router performs a standard top-$K$ selection per token. 
The selected experts and their normalized softmax weights define the routing decision.

\paragraph{Omitted components.}
For clarity, \Cref{lst:AIRrouter_pseudocode} omits several implementation details that are present in the 
full version used in experiments:
\begin{itemize}[leftmargin=*]
    \item load-balancing losses for experts,
    \item jitter noise for exploration,
    \item caching of shortlist-to-expert assignments,
    \item chunked computation for memory efficiency,
    \item overlap estimation with exact routing.
\end{itemize}

These components do not change the conceptual routing mechanism, but are included in the full implementation 
for improved training stability and scalability.

\begin{breakablepython}
import torch
import torch.nn as nn


class AIRRouter(nn.Module):
    """Minimal PyTorch-like AIR-Router."""
    def __init__(self, hidden_size, num_experts, num_shortlists, top_k, experts_per_shortlist):
        super().__init__()
        self.top_k = top_k
        self.experts_per_shortlist = experts_per_shortlist

        self.shortlist_codebook = AdaptiveKmeans(
            num_embed=num_shortlists,
            embed_dim=hidden_size
        )

        self.expert_features = nn.Parameter(
            torch.empty(num_experts, hidden_size)
        )
        nn.init.xavier_uniform_(self.expert_features)

        self._step = 0
        self._cached_step = None
        self._cached_experts_per_shortlist = None

    def on_optimizer_step(self):
        self._step += 1
        self.shortlist_codebook.on_optimizer_step()
        self._cached_step = None

    @torch.no_grad()
    def _get_experts_per_shortlist(self):
        if self._cached_step != self._step:
            centroids = self.shortlist_codebook.get_embedding_weights()
            scores = centroids @ self.expert_features.t()
            top = torch.topk(scores,
                             k=self.experts_per_shortlist,
                             dim=-1)
            self._cached_experts_per_shortlist = top.indices
            self._cached_step = self._step
        return self._cached_experts_per_shortlist

    def forward(self, hidden_states):
        B, L, H = hidden_states.shape
        hidden_states = hidden_states.reshape(-1, H)             # [S, H]

        # 1) Assign tokens to codewords and update the codebook.
        #    The cached expert shortlists are refreshed after optimizer steps.
        shortlist_idx = self.shortlist_codebook(hidden_states)  # [S]

        # 2) Reuse the current cached shortlist-to-expert table.
        experts_per_shortlist = self._get_experts_per_shortlist()     # [G, M]
        cand_experts = experts_per_shortlist[shortlist_idx]      # [S, M]

        # 3) Score candidates
        scores, cand_experts = self._score_candidates_grouped_mm(
            hidden_states,
            shortlist_idx,
            experts_per_shortlist,
        )
                              
        # 4) Select final top-K experts.
        top_scores, top_pos = torch.topk(scores,
                                         k=self.top_k,
                                         dim=-1)
        top_experts = cand_experts.gather(1, top_pos)            # [S, k]
        top_weights = torch.softmax(top_scores, dim=-1)          # [S, k]

        return top_weights, top_experts
\end{breakablepython}
\captionof{listing}{Simplified PyTorch-like code for the AIR-Router.}
\label{lst:AIRrouter_pseudocode}

\section{FLOP Analysis}

We estimate computational cost using \texttt{torch.utils.flop\_counter} and extend the default registry with custom formulas for operations that are either not covered or insufficiently modeled by PyTorch. All results in the paper report analytical FLOP counts obtained from traced operator graphs.

\paragraph{Counting convention.}
We count one floating-point addition, subtraction, multiplication, division, exponential, logarithm, or square root as one FLOP. Elementwise operators therefore contribute one FLOP per output element. Our goal is consistency across methods rather than exact hardware-level instruction counts.

\subsection{Elementwise and Reduction Operations}
For standard elementwise arithmetic (\texttt{add}, \texttt{sub}, \texttt{mul}, \texttt{div}) we count
\begin{equation}
\mathrm{FLOPs} = \mathrm{numel}(x).
\end{equation}
For common nonlinearities we use fixed per-element costs:
\begin{align}
\texttt{exp},\ \texttt{log},\ \texttt{sqrt},\ \texttt{rsqrt} &: 1,\\
\texttt{sigmoid},\ \texttt{silu} &: 3,\\
\texttt{gelu} &: 6.
\end{align}
Reductions are approximated as
\begin{align}
\texttt{mean} &: \mathrm{numel}(x) + 1,\\
\texttt{sum},\ \texttt{var\_mean} &: 2\,\mathrm{numel}(x).
\end{align}

\subsection{Softmax}
For softmax over dimension size $d_{\mathrm{sm}}$ with $N=\mathrm{numel}(x)$ we use
\begin{equation}
\mathrm{FLOPs} = 2N + \frac{N}{d_{\mathrm{sm}}},
\end{equation}
corresponding to one exponential and one division per element plus the reduction cost. The backward pass is approximated as $5\,\mathrm{numel}(x)$ FLOPs.

\subsection{Normalization Layers}
Let $d$ denote the normalized dimension and $V$ the number of normalized vectors.

\paragraph{LayerNorm.}
\begin{equation}
\mathrm{FLOPs} = V \,(4d + d + \mathbbm{1}_{\text{weight}} d + \mathbbm{1}_{\text{bias}} d).
\end{equation}
The backward pass is approximated as $8 V d$.

\paragraph{RMSNorm.}
\begin{equation}
\mathrm{FLOPs} = V \,(4d + \mathbbm{1}_{\text{weight}} d),
\end{equation}
with the same backward approximation $8 V d$.

\paragraph{BatchNorm.}
For $N=\mathrm{numel}(x)$ and $C$ channels we use
\begin{equation}
\mathrm{FLOPs} = (2N + 2C) + pN,
\end{equation}
where $p=2$ plus optional affine operations.

\subsection{Attention}
For scaled dot-product attention with query shape $(b,h,s_q,d)$ and key length $s_k$, we count
\begin{equation}
\mathrm{FLOPs}
= 4 b h s_q s_k d
+ 2 b h s_q s_k ,
\end{equation}
covering $QK^\top$, softmax, and attention--value multiplication.  
This formula is applied to both standard and fused attention kernels.

\subsection{Routing and Indexing Operations}
We also account for routing-related operators.

\paragraph{top-$K$.}
For selection over size $n$ with batch size $B$:
\begin{equation}
\mathrm{FLOPs} = B\, n \log_2(k+1),
\end{equation}
corresponding to an $O(n\log k)$ approximation.

\paragraph{Gather and scatter.}
For \texttt{gather}, \texttt{index\_add}, \texttt{scatter}, and related operators, we count one unit per affected element:
\begin{equation}
\mathrm{FLOPs} = \mathrm{numel}(\text{affected elements}).
\end{equation}

\section{Experimental Setup} \label{appx:setup}

\subsection{Data} \label{appx:data}

\textbf{WikiText-103}~\citep{merity2016pointer} is a language modeling dataset derived from
verified \emph{Good} and \emph{Featured} articles on Wikipedia. It contains
approximately $103$ million tokens in total, with about $103$M tokens for training,
$217$K for validation, and $245$K for testing. The dataset preserves the original
article structure and long-range dependencies, making it suitable for evaluating
long-context language modeling performance.

\textbf{C5}~\citep{raffel2020exploring} is a large-scale web text corpus constructed as an extension of the
Colossal Clean Crawled Corpus (C4). It contains approximately $87$ billion tokens after filtering and deduplication. Similar to C4, the data is obtained from Common Crawl and cleaned using heuristic filtering rules
to remove low-quality, boilerplate, and non-natural language content.

\textbf{OpenWebText2}~\citep{gao2020OpenWebText2,biderman2022datasheet} is a large-scale English web text corpus and a major component of \textit{The OpenWebText2}. It comprises approximately $65$~GiB of raw text (corresponding to roughly $15$--$20$ billion tokens), collected from outbound Reddit links and filtered to resemble the distribution of high-quality web text.

\subsection{Training}

All runs use AdamW with $\beta_1=0.9$, $\beta_2=0.95$, $\epsilon=10^{-8}$,
weight decay $0.1$, gradient clipping at $1.0$, seed $42$, and best-checkpoint
selection by validation loss (lower is better).

\resizebox{\textwidth}{!}{
\centering
\small
\setlength{\tabcolsep}{4pt}
\begin{tabular}{llrrrrrrlllrl}
\hline
Dataset & Size/Config & Block & Eff.\ BS/dev & Eval Strat. & Eval Freq. & LR Sched. & Peak LR & Prec. \\
\hline
OpenWebText2 & small  & 1024 & 16 & steps & 500 steps & linear (5\% warmup) & $3\mathrm{e}{-4}$ & bf16 \\
OpenWebText2 & medium & 1024 & 16 & steps & 500 steps & linear (5\% warmup) & $3\mathrm{e}{-4}$ & bf16 \\
OpenWebText2 & large  & 1024 & 32 & steps & 500 steps & linear (5\% warmup) & $3\mathrm{e}{-4}$ & bf16 \\
\hline
C5 & small  & 1024 & 16 & steps & 500 steps & linear (5\% warmup) & $3\mathrm{e}{-4}$ & bf16 \\
C5 & medium & 1024 & 16 & steps & 500 steps & linear (5\% warmup) & $3\mathrm{e}{-4}$ & bf16 \\
C5 & large  & 1024 & 32 & steps & 500 steps & linear (5\% warmup) & $3\mathrm{e}{-4}$ & bf16 \\
\hline
WikiText-103 & small  & 256 & 16 & epoch & every epoch & linear (5\% warmup) & $3\mathrm{e}{-4}$ & bf16 \\
WikiText-103 & medium & 256 & 16 & epoch & every epoch & linear (5\% warmup) & $3\mathrm{e}{-4}$ & bf16 \\
\hline
\end{tabular}
}

\subsection{Architectures}

Across all three scales, we maintain a consistent foundational setup to ensure fair comparisons. Every model utilizes the LLaMA-3 tokenizer with a vocabulary size of 128,256. We employ Rotary Positional Embeddings~(\citet{su2024roformer};~RoPE) with a base frequency of $\theta = 500,000$ to handle positional information, and we tie the token embedding weights with the pre-softmax output projection to improve parameter efficiency. Furthermore, each variant leverages Grouped-Query Attention~(\citet{ainslie2023gqa};~GQA), optimizing inference throughput without sacrificing representational power. 

The primary axes of scaling across our models are depth (number of hidden layers), width (hidden dimension $d_{\text{model}}$ and FFN intermediate dimension $d_{\text{ffn}}$), and attention capacity. The specific hyperparameter configurations for each scale are detailed below:

\begin{center}
\renewcommand{\arraystretch}{1.2}
\begin{tabular}{l c c c}
    \toprule
    \textbf{Hyperparameter} & \textbf{Small} & \textbf{Medium} & \textbf{Large} \\
    \midrule
    Hidden Dimension ($d_{\text{model}}$) & 256 & 512 & 768 \\
    Intermediate Dimension ($d_{\text{ffn}}$) & 768  & 1536 & 2048 \\
    Number of Layers & 16 & 24 & 24 \\
    Attention Heads ($n_{\text{heads}}$) & 4 & 8 & 12 \\
    Key/Value Heads ($n_{\text{kv\_heads}}$) & 1 & 2 & 4 \\
    \midrule
    Vocabulary Size & \multicolumn{3}{c}{128,256} \\
    Max Sequence Length & \multicolumn{3}{c}{2048} \\
    RoPE Base ($\theta$) & \multicolumn{3}{c}{$5\times 10^{5}$} \\
    Tied Embeddings & \multicolumn{3}{c}{True} \\
    \bottomrule
\end{tabular}
\end{center}

\subsection{Efficiency at Scale Experiments} \label{appx:scaling_setup}

The following setup is used for~\cref{sec:scale_num_experts} and~\cref{appx:beyond_flops}: Unless stated otherwise, the experiments use the standard configuration corresponding to the large-architecture setting and a batch size of $32$ (corresponding to $32768$ tokens), with $E=1024^2$ for the middle panels and $E=256^2$ for the bottom panel. We use the smaller expert count in the latter case because the standard granular baseline runs out of memory even with routing dimension $2$. The experiments are run on a single NVIDIA H100 80GB HBM3 GPU. All router implementations rely exclusively on existing PyTorch operations and \textit{do not use custom kernel implementations}. Further performance improvements may therefore be achievable through dedicated kernel optimization.

\subsection{Hyperparameters} \label{appx:hyperparameters}

\subsubsection{Scale-invariant Hyperparameters}
\paragraph{\ourmethod{}.} We set $\gamma=0.95$, $\tau=1.0$, jitter standard deviation $\sigma_{\mathrm{jit}}=0.005$, load balancing weight $\lambda=10^{-4}$ for C5 and wikitext, and load balancing weight $\lambda=10^{-5}$ for OpenWebText2.

\paragraph{Hierarchical Router.} Analogously to \ourmethod{}, we set $\sigma_{\mathrm{jit}}=0.005$ and apply it at both the cluster selection and expert selection level. We furthermore apply identical load balancing as for \ourmethod{}.

\paragraph{Standard Coarse.} Across all experiments, we set $\sigma_{\mathrm{jit}}=0.005$ and use dropless expert-choice gating.

\subsubsection{MoE Configurations} \label{appx:baselines}

To ensure fair and meaningful comparisons across all MoE routing methods, we enforce a
set of rules that govern how each method's hyperparameters are set relative to one
another. These rules are applied uniformly across all model sizes.

\paragraph{Notation.}
Let $E$ denote the total number of experts, $K$ the number of experts activated per token
(\emph{active-$K$}), and $l$ the number of coarse-level candidates selected before fine-grained
routing (needed for the hierarchical router, \ourmethod{} always has $l=1$). For the PEER router, which maintains $P$ independent routing heads, the effective active-$K$ is $K \cdot P$. In practice, in line with the original paper~\citep{he2024mixture}, we choose $P=8$ and choose $K$ such that the total amount of active experts matches the amount of active experts of \ourmethod{} (that is, $K_{\mathrm{PEER}} = \frac{K_{\mathrm{AIR}}}{P}$). We also use query BatchNorm~\citep{lample2019large} in PEER, as described by~\citet{he2024mixture}. For each model size, we use the same with a fixed effective active-$K$ budget, specifically $K=128$. 

\paragraph{Choice of $M$ and $G$ for \ourmethod{}.} As motivated empirically (see~\cref{appx:codebook_size_vs_shortlist_size}), we generally choose $M$ and $G$ as a function of $E$ according to the following rule
\begin{equation*}
    G \, = \, \frac{\sqrt{E}}{2}, \quad M = \frac{4E}{G}.
\end{equation*}

\paragraph{Cross-method fairness.}
All methods compared within the same experiment are constrained to be equivalent along
three dimensions:

\begin{enumerate}[label=\textbf{\Alph*.}, leftmargin=*]
  \item \textbf{Equal active-$K$.} The effective number of experts activated per token is
    held fixed across all granular methods. For PEER this is $K \cdot P$; for the standard coarse router it is $K$ directly. This ensures that the total compute per token (in terms of
    expert FLOPs) is comparable.

  \item \textbf{Equal coarse structure.} The AIR router uses $G$ codebook shortlists and the
    hierarchical router uses $G$ clusters. We enforce this to be equal so that the coarseness
    of the two-stage selection process is matched.

  \item \textbf{Equal candidate pool.} Beyond equal active-$K$, we additionally match the
    number of candidate experts scored before final selection: $|\mathcal{E}_\text{shortlist}|$
    for AIR and $l \cdot (E / G)$ for hierarchical routing.
\end{enumerate}

\paragraph{Coarse baseline equivalence.}
The standard coarse MoE baseline (top-1 routing over large experts) is constructed to be
parameter-equivalent to the granular model. Specifically, given granular active-$K$ and
$E$ granular experts each of intermediate size $1$, the coarse model uses $K=1$ with
$E_\text{coarse} = \lceil E / K \rceil$ experts each of intermediate size $K$. This
ensures the active parameter count per forward pass is matched.

\subsection{Hardware} \label{appx:hardware}

All experiments are run on an internal HTCondor cluster. The individual experiments are run on the following hardware setups:

\resizebox{\textwidth}{!}{
\centering
\begin{tabular}{llrrrrrr}
\hline
Dataset & Size/Config & GPU Name & \# GPUs & \# CPUs & \# processes & host RAM  \\
\hline
OpenWebText2 & small & NVIDIA H100 80GB HBM3 & 4 & 16 & 4 & 120 \\
\hline
OpenWebText2 & medium & NVIDIA H100 80GB HBM3 & 4 & 16 & 4 & 120 \\
\hline
OpenWebText2 & large & NVIDIA H100 80GB HBM3 & 4 & 32 & 4 & 240 \\
\hline
C5 & small & NVIDIA H100 80GB HBM3 & 4 & 16 & 4 & 200 \\
\hline
C5 & medium & NVIDIA H100 80GB HBM3 & 4 & 16 & 4 & 200 \\
\hline
C5 & large & NVIDIA H100 80GB HBM3 & 4 & 32 & 4 & 240 \\
\hline
WikiText-103 & small & NVIDIA H100 80GB HBM3 & 4 & 16 & 4 & 120 \\
\hline
WikiText-103 & medium & NVIDIA H100 80GB HBM3 & 4 & 16 & 4 & 120 \\
\hline
\end{tabular}
}

\section{Additional Experimental Results} \label{appx:additional_experiments}

\subsection{Multi-Seed Runs} \label{appx:multiseed_runs}

To provide evidence that the performance improvement of \ourmethod{} in comparison to existing methods is invariant to the choice of seed, we run two additional seeds for the medium-sized AIR and PEER configurations on C5 and OpenWebText2. The results are given in Table~\ref{tab:multiseed_runs}.

\begin{table}[ht]
    \centering
    \caption{Results across three random seeds for the medium-sized AIR and PEER configurations on C5 and OpenWebText2. Relative improvements are computed with respect to PEER for the corresponding seed.}
    \label{tab:multiseed_runs}
    \begin{tabular}{llrr}
        \toprule
        Seed & Method & PPL $\downarrow$ & Rel.\ Impr. $\uparrow$ \\
        \midrule
        \multicolumn{4}{l}{\textbf{C5}} \\
        \addlinespace[2pt]
        Seed 42 & AIR  & \textbf{43.85} & \textbf{+4.4\%} \\
                & PEER & 45.88          & --             \\
        \addlinespace
        Seed 43 & AIR  & \textbf{24.75} & \textbf{+3.0\%} \\
                & PEER & 25.53          & --             \\
        \addlinespace
        Seed 44 & AIR  & \textbf{32.02} & \textbf{+1.2\%} \\
                & PEER & 32.40          & --             \\
        \midrule
        \multicolumn{4}{l}{\textbf{OpenWebText2}} \\
        \addlinespace[2pt]
        Seed 42 & AIR  & \textbf{20.64} & \textbf{+3.9\%} \\
                & PEER & 21.49          & --             \\
        \addlinespace
        Seed 43 & AIR  & \textbf{21.40} & \textbf{+3.0\%} \\
                & PEER & 22.06          & --             \\
        \addlinespace
        Seed 44 & AIR  & \textbf{21.49} & \textbf{+3.3\%} \\
                & PEER & 22.23          & --             \\
        \bottomrule
    \end{tabular}
\end{table}

Across all three seeds and both datasets, AIR consistently outperforms PEER, demonstrating that the observed performance gains are robust to the choice of random seed.

\subsection{Additional Metrics} \label{appx:additional_metrics}

The additional metrics, shown in~\cref{tab:combined_baselines}, provide more information about \ourmethod{}: It comes at a similar FLOP cost to PEER, but slightly more than the hierarchical router due to the necessity to refresh the shortlist cache.

\begin{longtable}{l c c c c}
\caption{Results across datasets and model sizes.}\label{tab:combined_baselines}\\
\toprule
Method & PPL $\downarrow$ & FLOPs $\downarrow$ & Dead Experts $\downarrow$ & Entropy $\uparrow$ \\
\midrule
\endfirsthead

\toprule
Method & PPL $\downarrow$ & FLOPs $\downarrow$ & Dead Experts $\downarrow$ & Entropy $\uparrow$ \\
\midrule
\endhead

\midrule
\multicolumn{5}{r}{\emph{Continued on next page}} \\
\endfoot

\bottomrule
\endlastfoot

\multicolumn{5}{l}{\textbf{Dataset: wikitext-103 \;|\; Size: small}} \\
\midrule
\underline{AIR} & \textbf{19.79} & 363.4P & 51.3\% & 11.39 \\
PEER & 20.30 & 364.4P & 8.0\% & 11.02 \\
Hierarchical & 20.21 & \textbf{322.8P} & \textbf{0.0\%} & \textbf{12.06} \\
Std.\ Coarse & 21.79 & 396.7P & 0.1\% & 6.97 \\
\addlinespace[0.5em]
\multicolumn{5}{l}{\textbf{Dataset: wikitext-103 \;|\; Size: medium}} \\
\midrule
\underline{AIR} & \textbf{18.92} & 370.7P & 19.3\% & 12.50 \\
PEER & 19.64 & 366.5P & 53.1\% & 10.90 \\
Hierarchical & 24.05 & 364.8P & 1.0\% & \textbf{12.66} \\
Std.\ Coarse & 21.02 & \textbf{364.1P} & \textbf{0.9\%} & 7.95 \\
\addlinespace[0.5em]
\multicolumn{5}{l}{\textbf{Dataset: c5 \;|\; Size: small}} \\
\midrule
\underline{AIR} & \textbf{546.87} & 241.3P & 38.8\% & 8.80 \\
PEER & 570.88 & 241.9P & \textbf{0.0\%} & \textbf{9.45} \\
Hierarchical & 918.47 & 241.1P & \textbf{0.0\%} & 8.66 \\
Std.\ Coarse & 796.31 & \textbf{237.0P} & \textbf{0.0\%} & 6.43 \\
\addlinespace[0.5em]
\multicolumn{5}{l}{\textbf{Dataset: c5 \;|\; Size: medium}} \\
\midrule
\underline{AIR} & \textbf{43.85} & 2.4E & 13.7\% & 10.48 \\
PEER & 45.88 & 2.4E & 3.5\% & \textbf{10.53} \\
Hierarchical & 51.45 & \textbf{2.3E} & 85.5\% & 8.93 \\
Std.\ Coarse & 49.81 & \textbf{2.3E} & \textbf{0.0\%} & 7.59 \\
\addlinespace[0.5em]
\multicolumn{5}{l}{\textbf{Dataset: c5 \;|\; Size: large}} \\
\midrule
\underline{AIR} & \textbf{46.84} & \textbf{7.8E} & 16.5\% & \textbf{10.67} \\
PEER & 47.86 & \textbf{7.8E} & 0.3\% & 10.03 \\
Hierarchical & 51.73 & \textbf{7.8E} & 97.9\% & 7.75 \\
Std.\ Coarse & 49.06 & \textbf{7.8E} & \textbf{0.0\%} & 7.35 \\
\addlinespace[0.5em]
\multicolumn{5}{l}{\textbf{Dataset: OpenWebText2 \;|\; Size: small}} \\
\midrule
\underline{AIR} & \textbf{49.16} & 219.3P & 63.8\% & 9.55 \\
PEER & 51.95 & 219.9P & 1.2\% & \textbf{10.16} \\
Hierarchical & 67.21 & 219.2P & \textbf{0.0\%} & 9.73 \\
Std.\ Coarse & 56.85 & \textbf{215.5P} & \textbf{0.0\%} & 6.67 \\
\addlinespace[0.5em]
\multicolumn{5}{l}{\textbf{Dataset: OpenWebText2 \;|\; Size: medium}} \\
\midrule
\underline{AIR} & \textbf{20.64} & 2.0E & 39.1\% & 10.48 \\
PEER & 21.49 & 2.0E & 53.8\% & 10.02 \\
Hierarchical & 25.68 & \textbf{1.8E} & 5.3\% & \textbf{10.83} \\
Std.\ Coarse & 22.77 & 2.0E & \textbf{0.0\%} & 7.69 \\
\addlinespace[0.5em]
\multicolumn{5}{l}{\textbf{Dataset: OpenWebText2 \;|\; Size: large}} \\
\midrule
\underline{AIR} & \textbf{18.03} & 6.0E & 34.4\% & 10.57 \\
PEER & 18.16 & 6.0E & 29.4\% & 10.19 \\
Hierarchical & 20.24 & \textbf{5.7E} & 2.3\% & \textbf{10.94} \\
Std.\ Coarse & 18.67 & 6.0E & \textbf{0.0\%} & 7.42 \\
\end{longtable}

\subsection{Qualitative Differences in Expert Usage} \label{appx:expert_usage}

\begin{figure}
    \centering
    \includegraphics[width=0.88\textwidth]{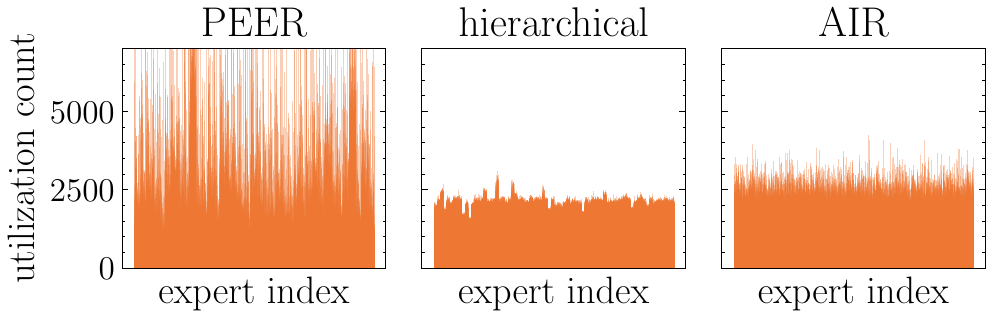}
    \caption{\textbf{Qualitative Comparison of Expert Usage.} Existing routing methods impose hard structural constraints on the expert selection mechanism, either via structural constraints (PEER) or grouping constraints. These constraints, in turn, lead to routing artifacts depicted in \textbf{(a)} and \textbf{(b)}, respectively. \ourmethod{}, in contrast, imposes no assumptions by computing an approximation of the top-$K$ expert scores whose quality is based on how well the internal representation cluster naturally. }
    \label{fig:expert_usage}
\end{figure}

The greatest difference between~\ourmethod{} and the considered baselines for granular MoE lies in the fact that \ourmethod{} does not impose restrictions, as can be seen in~\cref{fig:expert_usage}: Both PEER and the hierarchical router achieve FLOP-efficient routing via coupling neighboring expert indices via weight sharing (PEER) or grouping (hierarchical). \ourmethod{}, in contrast, does not impose such restrictions and therefore provides a more flexible model.

\subsection{Additional Efficiency at Scale Experiments} \label{appx:beyond_flops}
\revised{For the experimental setup, see~\cref{appx:scaling_setup}. The results, which can be seen in~\cref{fig:router_comparisons}, show that \ourmethod{} (blue) scales substantially better than the standard granular router (pink) and comparably to the hierarchical and PEER routers across all reported metrics as the microbatch size and routing dimension increase. For the latter, \ourmethod{} has slightly lower runtime than PEER.
}
\begin{figure}
    \centering
    \includegraphics[width=1\textwidth]{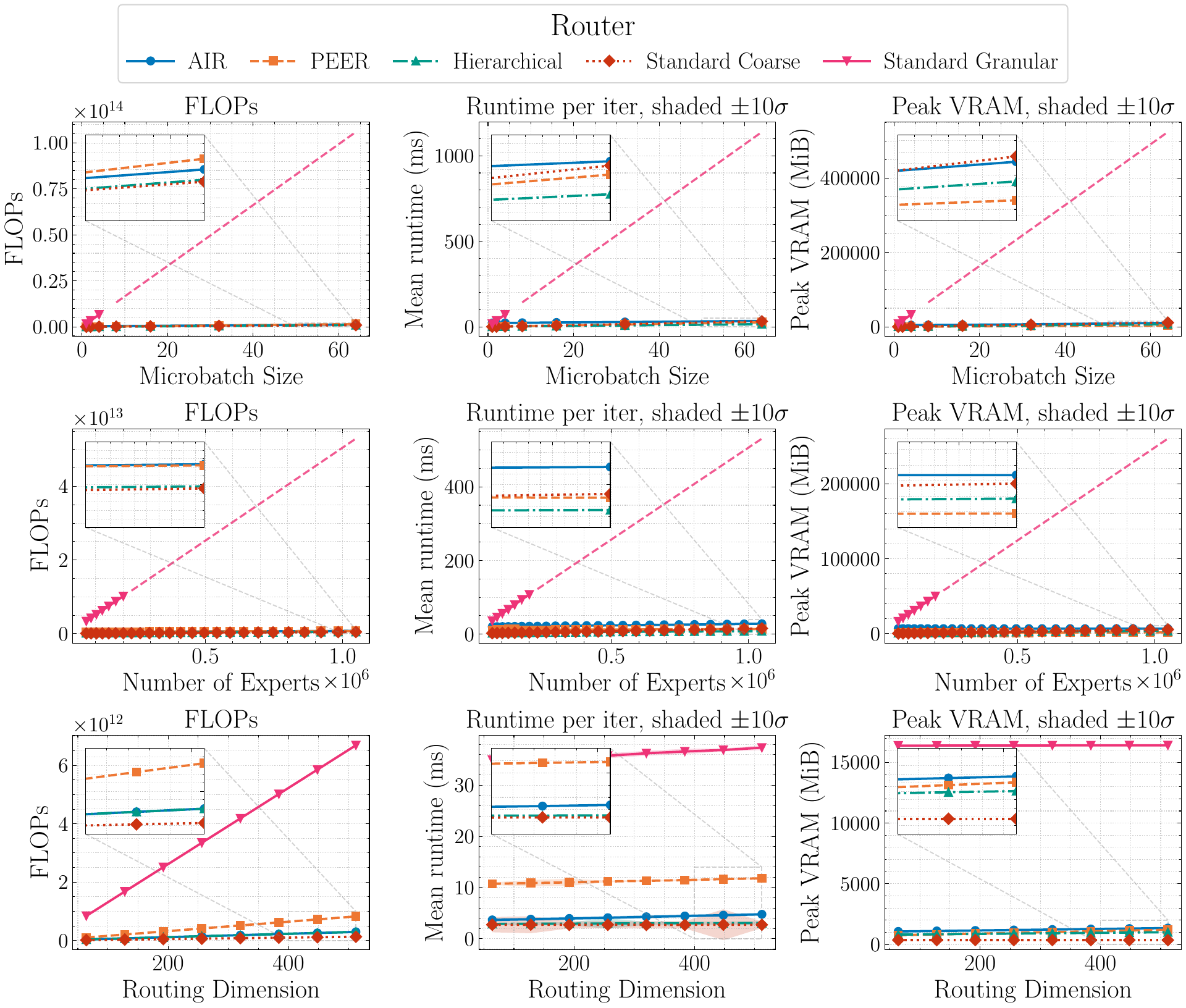}
    \caption{\revised{\textbf{All Efficiency at Scale Experiments.} Each row shows results of different routers as a function of different parameters $\pm 10 \sigma$, where $\sigma$ is standard deviation. The dashed line for the standard granular baseline corresponds to a linear extrapolation that had to be done due to memory overflow. The plots show that \ourmethod{} effectively reduces routing overhead in comparison to the standard granular routing approach beyond being FLOP-efficient.}}
    \label{fig:router_comparisons}
\end{figure}

\subsection{Ablation Studies} \label{appx:ablations}
We ablate key components of \ourmethod{} on the medium model using C5.

\paragraph{Ablations.}
We consider the following variants: (1) No centroid projection. We remove the projection of expert centroids onto the unit sphere in~\cref{alg:cf_router}, thereby decoupling routing from cosine geometry. (2) Non-adaptive codebook. The codebook is initialized from randomly selected training tokens and kept fixed, removing adaptivity to the evolving representation space. (3) Expert-choice gating. We replace standard load balancing~\citep{fedus2022switch} with expert-choice routing~\citep{zhou2022mixture}. Experts outside the shortlist receive a routing score of $-\infty$. (4) Cosine similarity assignment. We replace Euclidean $k$-means with spherical clustering~\citep{dhillon2001concept} in~\cref{alg:spherical_kmeans}.

\paragraph{Results.}
As shown in~\cref{tab:ablations}, \ourmethod{} achieves the best perplexity. Removing normalization has some negative effect on perplexity, while having no considerable effect in terms of FLOPs. Using cosine similarity slightly increases perplexity, but slightly increases expert usage and overlap at comparable FLOPs. Expert-choice gating attains the highest overlap, but at the cost of severe under-utilization ($68.1\%$ dead experts), showing that overlap alone is not a reliable proxy for effective routing. The static codebook achieves the second-worst overall results, confirming the effectiveness of the adaptive codebook.

\begin{table} \centering \caption{\textbf{Ablation results.} Results shown for C5, medium model. The results confirm that the different components of \ourmethod{} are relevant.} \label{tab:ablations} 
\begin{tabular}{l c c c c c}
\toprule
Ablation & PPL $\downarrow$ & FLOPs $\downarrow$ & Dead Experts $\downarrow$ & Overlap Approx. $\uparrow$ & Entropy $\uparrow$ \\
\midrule
\underline{AIR} & \textbf{43.85} & \textbf{2.4E} & 13.7\% & 0.58 & 10.48 \\
Cosine & 44.11 & \textbf{2.4E} & \textbf{8.4\%} & 0.65 & 10.45 \\
Expert Choice & 48.77 & \textbf{2.4E} & 68.1\% & \textbf{0.81} & \textbf{11.95} \\
Missing Normalization & 46.20 & \textbf{2.4E} & 85.9\% & 0.72 & 9.85 \\
Static Code & 47.81 & \textbf{2.4E} & 89.1\% & 0.05 & 10.01 \\
\bottomrule
\end{tabular}
\end{table}

\subsection{Empirical Inspection of Shortlist Overlap} \label{appx:overlap_analysis}
In addition to dynamic grouping of experts, another major difference between \ourmethod{} and hierarchical routing architectures~\citep{shazeer2017outrageously} lies in the fact that expert groups are generally overlapping, which we deem to be a major contributor to the improved performance of \ourmethod{} in comparison to the hierarchical approach. In this section, we provide some insight into how experts empirically spread across shortlists.
\\

\revised{
As can be seen in~\cref{fig:shortlist_overlap}, for the large model trained on C5 with $E=1024^2$, the distribution over shortlist coverages is highly skewed: while most experts appear only in a small number of shortlists, a smaller number of experts are included much more frequently. This suggests that the learned shortlists can naturally accommodate broadly useful ``hot'' experts. We view this as a desirable property, since some experts may learn transformations that are useful across multiple regions of the input space.~\footnote{The importance of such experts has been explored in earlier compression works by~\citet{li2024merge, li2024quantmoe}} Overlapping shortlists allow such experts to be reused flexibly, without forcing all tokens that benefit from them to pass through the same coarse routing decision. Hierarchical routers, in contrast, impose a more rigid partitioning structure: every expert is assigned to one particular branch and it can only be reached through that branch, which limits the sharing of broadly useful experts across otherwise distinct groups of tokens.
}

\begin{figure} \label{fig:mask_overlap}
    \centering
    \includegraphics[width=0.8\textwidth]{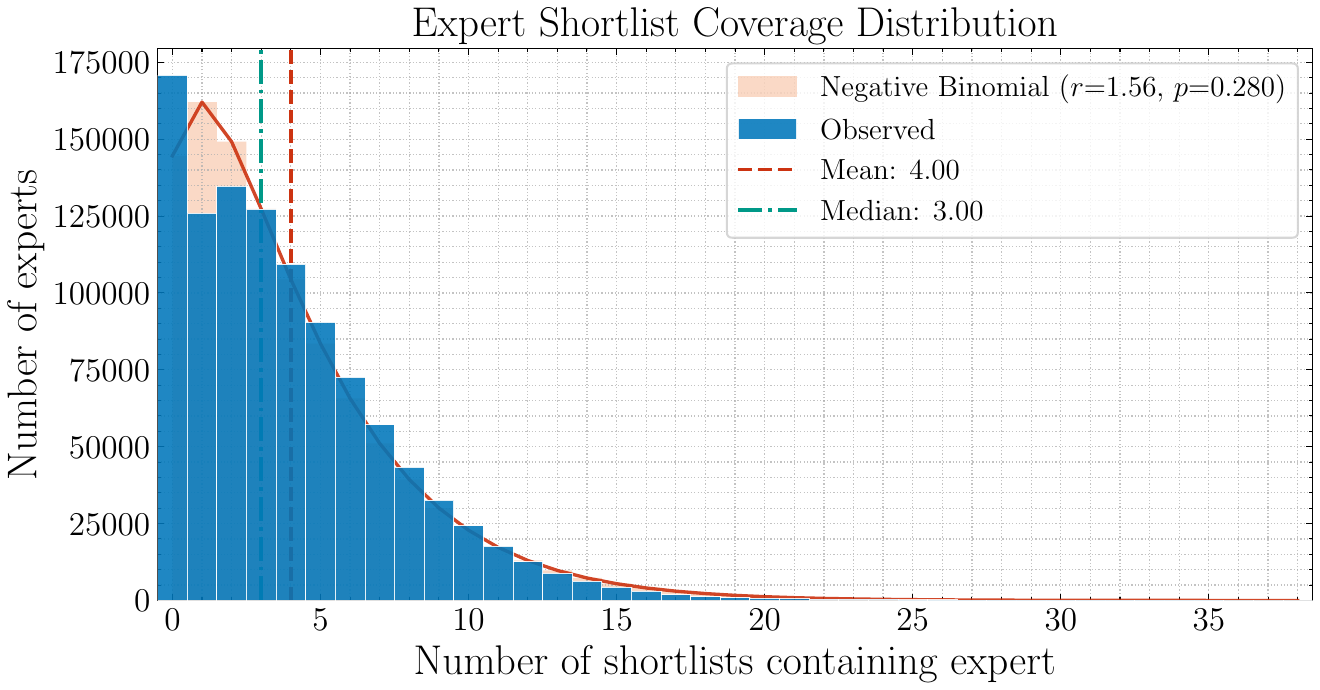}
    \caption{\revised{\textbf{Shortlist Overlap Analysis.} Histogram about how many times experts are covered my different shortlists. The results shows that the distribution is highly skewed and has a similar shape to a negative binomial distribution (plotted in the background), which reflects the fact that \ourmethod{} actively models ``hot'' experts that represent broadly useful transformations.}}
    \label{fig:shortlist_overlap}
\end{figure}

\subsection{On the Choice of Codebook Size and Shortlist Size} \label{appx:codebook_size_vs_shortlist_size}
\begin{figure}
    \centering
    \includegraphics[width=1\textwidth]{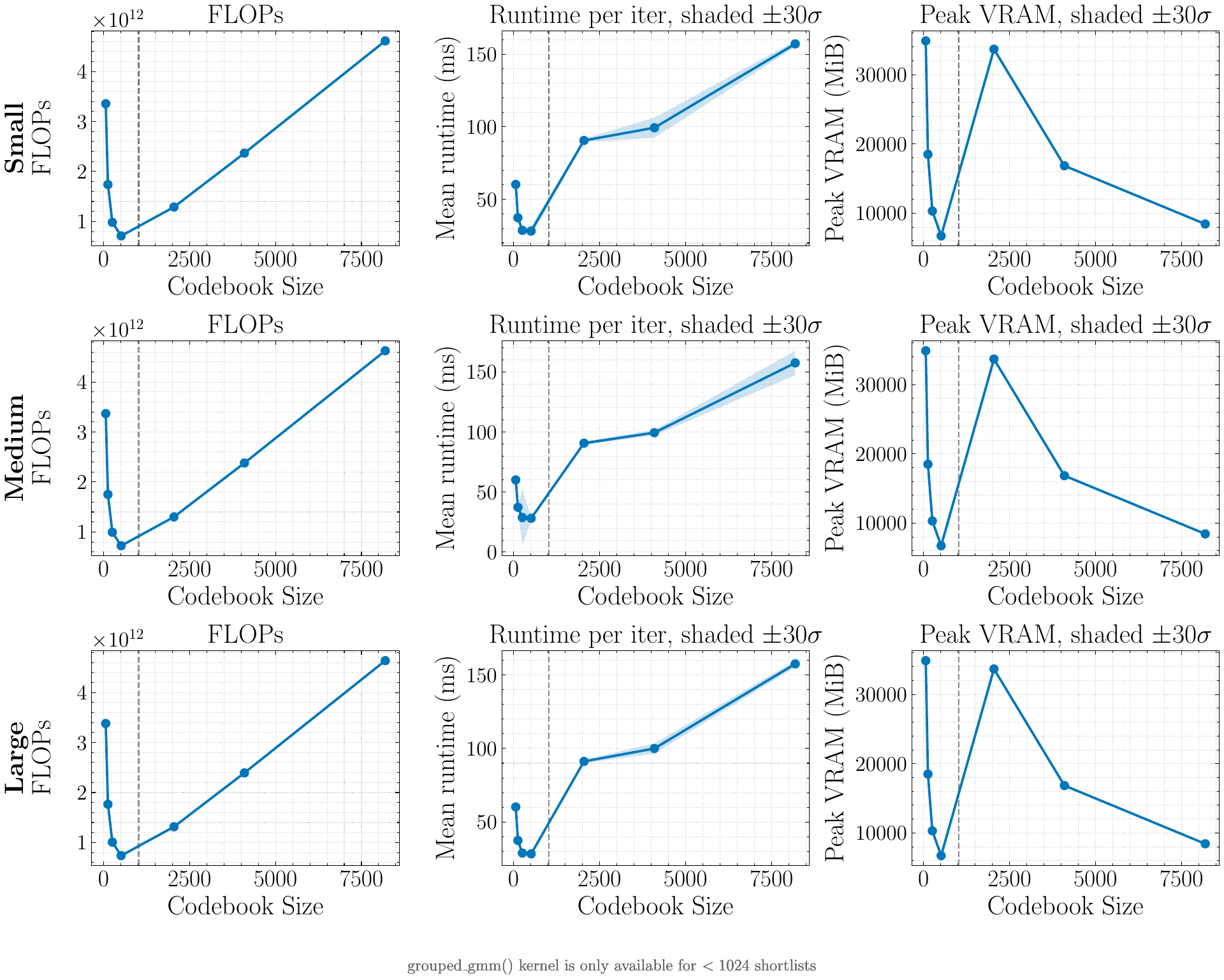}
    \caption{\revised{\textbf{Performance of \ourmethod{} for Different Codebook Sizes.} Each row reports results for a different model size, parameterized by the hidden dimension. For each codebook size $G$, the shortlist size $M$ is chosen such that $G \cdot M = 4E$. Overall, smaller codebooks yield more efficient routing. The kink at a codebook size of $1024$ arises because the \texttt{grouped\_mm()} kernel supports dimensions only up to $1024$. For larger codebooks, we therefore use a more memory-bandwidth-intensive implementation that gathers the corresponding expert centroids separately for each token.}}
    \label{fig:experts-per-mask_vs_num-masks}
\end{figure}
It is easy to see that increasing the amount of experts per shortlist $M$ has a similar effect to increasing the codebook size $G$ (i.e., the amount of shortlists). In order to find a good trade-off between the two, we set a fixed budget $G \cdot M = 4 E$, where we recall that $E$ is the total amount of experts. We then vary the value of $M$ as we assess different values of $G.$

The results, which can be seen in~\cref{fig:experts-per-mask_vs_num-masks}, show that there exists a trade-off between codebook size and shortlist size: if the codebook is very small, the codebook update step is cheap, but the cost of evaluating token-expert scores increases. If the codebook is too large, however, the shortlist update dominates overall routing cost. In addition, the \texttt{grouped\_mm()} kernel is only supported for dimensions up to 1024. Therefore, larger codebook sizes are implemented via replicating all expert centroids per token (no grouping), which leads to high memory demand and increased memory latency (leading to the observed kink) for large shortlists. We find that codebook size $G=512$ for $E=1024^2$ experts provides the best trade-off among the tested configurations.

\end{document}